\documentclass[11pt]{article}
  
\usepackage{subfigure} 

\usepackage{mathtools}
\usepackage{fullpage} %package to reduce margins
\usepackage{algorithm}
\usepackage{algorithmic}
\usepackage{amssymb}
\usepackage{epsfig}
\usepackage{psfrag}
\usepackage{latexsym}
\usepackage{amsmath,amsthm,bm, color}
\def\T={\buildrel {\scriptscriptstyle\triangle} \over =}

\def\ba{\begin{array}}
\def\ea{\end{array}}

\newcommand{\Sa}{S_{i_{1}^{*}}^{(1)}}
\newcommand{\Sb}{S_{i_{2}^{*}}^{(1)}}
\newcommand{\Sc}{S_{k_{1}^{*}}^{(3)}}
\newcommand{\Sd}{S_{k_{2}^{*}}^{(3)}}
\newcommand{\Xa}{X_{i_{1}^{*}}^{1}}
\newcommand{\Xb}{X_{i_{2}^{*}}^{1}}
\newcommand{\Xc}{X_{k_{1}^{*}}^{3}}
\newcommand{\Xd}{X_{k_{2}^{*}}^{3}}
\newcommand{\Ma}{m^{(1)}_{i_{1}^{*}} }
\newcommand{\Mb}{m^{(1)}_{i_{2}^{*}}}
\newcommand{\Mc}{m^{(3)}_{k_{1}^{*}}}
\newcommand{\Md}{m^{(3)}_{k_{2}^{*}}}

\newcommand{\R}{\mathbb{R}}

\newcommand{\vct}[1]{\bm{#1}}
\newcommand{\mtx}[1]{\bm{#1}}
\newcommand{\vX}{\vct{X}}

\newcommand{\vW}{\vct{W}}

\newcommand{\vA}{\vct{A}}
\newcommand{\vB}{\vct{B}}
\newcommand{\vC}{\vct{C}}
\newcommand{\vD}{\vct{D}}
\newcommand{\vT}{\vct{T}}

\newtheorem{theorem}{Theorem}[section]
\newtheorem{lemma}[theorem]{Lemma}
\newtheorem{assumption}{Assumption}[section]

\newtheorem{corollary}[theorem]{Corollary}
\newtheorem{definition}{Definition}[section]
\newtheorem*{metatheorem}{Meta-Theorem}

\newenvironment{remark}[1][Remark]{\begin{trivlist}
\item[\hskip \labelsep {\bfseries #1}]}{\end{trivlist}}
\newtheorem*{remarks}{Remarks}

\begin{document}

%\icmladdress{Colorado School of Mines}

%\icmlkeywords{manifold learning,  signal recovery, 3machine learning,}

%\vskip 0.3in
\title{Optimal Low-Rank Tensor Recovery from Separable Measurements: Four Contractions Suffice}

\maketitle

% unless you've provided the [accepted] option to the icml2011
\begin{center}
\author{Parikshit Shah}\footnote{Yahoo! Labs \quad parikshit@yahoo-inc.com} \qquad
%\icmladdress{Wisconsin Institutes of Discovery}
\author{Nikhil Rao}\footnote{UT Austin \quad nikhilr@cs.utexas.edu} \qquad
%\icmladdress{University of Texas, Austin}
\author{Gongguo Tang}\footnote{Colorado School of Mines \quad gtang@mines.edu}
\end{center}
%%%%%%%%%%%%%%%%%%%%%%%%%%%%%%%%%%%%%%%%%%%%%%%%
%%%%%%%%%%%%%%%%%%%%%%%%%%%%%%%%%%%%%%%%%%%%%%%%
%%%%%%%%%%%%%%%%%%%%%%%%%%%%%%%%%%%%%%%%%%%%%%%%
%%%%%%%%%%%%%%%%%%%%%%%%%%%%%%%%%%%%%%%%%%%%%%%%
%%%%%%%%%%%%%%%%%%%%%%%%%%%%%%%%%%%%%%%%%%%%%%%%
%%%%%%%%%%%%%%%%%%%%%%%%%%%%%%%%%%%%%%%%%%%%%%%%
%%%%%%%%%%%%%%%%%%%%%%%%%%%%%%%%%%%%%%%%%%%%%%%%
%%%%%%%%%%%%%%%%%%%%%%%%%%%%%%%%%%%%%%%%%%%%%%%%
%%%%%%%%%%%%%%%%%%%%%%%%%%%%%%%%%%%%%%%%%%%%%%%%
%%%%%%%%%%%%%%%%%%%%%%%%%%%%%%%%%%%%%%%%%%%%%%%%

\begin{abstract} 
Tensors play a central role in many modern machine learning and signal processing applications. In such applications, the target tensor is usually of low rank, i.e., can be expressed as a sum of a small number of rank one tensors. This motivates us to consider the problem of low rank tensor recovery from a class of linear measurements called \emph{separable measurements}. As specific examples, we focus on two distinct types of separable measurement mechanisms (a) Random projections, where each measurement corresponds to an inner product of the tensor with a suitable random tensor, and (b) the completion problem where measurements constitute revelation of a random set of entries. We present a computationally efficient algorithm, with rigorous and \emph{order-optimal} sample complexity results (upto logarithmic factors) for tensor recovery. Our method is based on reduction to matrix completion sub-problems and adaptation of Leurgans' method for tensor decomposition. We extend the methodology and sample complexity results to higher order tensors, and experimentally validate our theoretical results

%UAI: 
%
%Motivated by modern applications of tensors in machine learning, we consider the problem of tensor completion - i.e. recovering a low-rank tensor from partially revealed entries. When the underlying tensor is of dimension $n_1 \times n_2 \times n_3$ (where $n_3$ is the largest dimension) and rank $r \leq n_1$, we show that (under mild assumptions) $O(rn_3 \log^2 n_3)$ measurements suffice to provably and exactly recover the underlying tensor - a near-optimal sample complexity. The key observation that drives our method is that samples from only \emph{four slices of the tensor suffice}, i.e. if there are enough measurements in four different slices of the tensor to solve the related matrix completion problems, then the full tensor can be pieced together efficiently by spectral methods. Our algorithm relies on a reduction to matrix completion sub-problems for tensor slices and an adaptation of Leurgans' method for tensor decomposition, and is therefore computationally efficient and scalable, unlike competing approaches proposed recently in the literature. We extend the methodology and sample complexity results to higher order tensors, and experimentally validate our theoretical results.
\end{abstract} 

%%%%%%%%%%%%%%%%%%%%%%%%%%%%%%%%%%%%%%%%%%%%%%%%
%%%%%%%%%%%%%%%%%%%%%%%%%%%%%%%%%%%%%%%%%%%%%%%%
%%%%%%%%%%%%%%%%%%%%%%%%%%%%%%%%%%%%%%%%%%%%%%%%
%%%%%%%%%%%%%%%%%%%%%%%%%%%%%%%%%%%%%%%%%%%%%%%%
%%%%%%%%%%%%%%%%%%%%%%%%%%%%%%%%%%%%%%%%%%%%%%%%
%%%%%%%%%%%%%%%%%%%%%%%%%%%%%%%%%%%%%%%%%%%%%%%%
%%%%%%%%%%%%%%%%%%%%%%%%%%%%%%%%%%%%%%%%%%%%%%%%
%%%%%%%%%%%%%%%%%%%%%%%%%%%%%%%%%%%%%%%%%%%%%%%%
%%%%%%%%%%%%%%%%%%%%%%%%%%%%%%%%%%%%%%%%%%%%%%%%
%%%%%%%%%%%%%%%%%%%%%%%%%%%%%%%%%%%%%%%%%%%%%%%%

%-------------------------------------------------
\section{Introduction} \label{sec:intro}
%Introduce notion of contraction early
% Consilidate algos - Leurgans and then sensing and completion
% Throughout consistently use random sensing and completion
% Motivate Gaussian sensing better - what do the measurements mean? a CandesTao type example?
% Any connections to Radon transform possible?
% Connections to matrix recovery: notion of contractions
% Look up RFP, CR, CS off grid etc. papers for ideas to improve intro
% Mention applications of tensor completion - recsys, multiview models with missing data, etc.
% Look up intro of other tensor papers
% Which notion of tensor rank and why
% low rank (rank smaller than dimension)
% tensors hard, decomp generically easy
% look up results in Hillar&Lim
% Our algorithm is simple and relatively low complexity
% Reduces to mat comp, and to make it scalable other fast heuristics may be employed - algo will work if heuristics work
% Instead of JMLR, how about FOCM? Or Applied and Comp. Harmonics?

Tensors provide compact representations for multi-dimensional, multi-perspective data in many problem domains, including image and video processing \cite{Zhang:2008ey, Liu:2013bh, {Kang:2002iu}}, collaborative filtering \cite{Karatzoglou:2010bm,Chen:2005jn}, statistical modeling \cite{Anandkumar:2014vz, animatensor}, array signal processing \cite{Lim:2010if, Sidiropoulos:2000gb}, psychometrics \cite{Wold:1987iw,Smilde:2005wf},  neuroscience \cite{Beckmann:2005fg, MartnezMontes:2004ic}, and large-scale data analysis \cite{Papalexakis:2013tu, Sun:2009tc, Sun:2006ga, Anandkumar:2013va, Cichocki:2014vy}. In this paper we consider the problem of tensor recovery - given partial information of a tensor via linear measurements, one wishes to learn the entire tensor. While this inverse problem is ill-posed in general, we will focus on the setting where the underlying tensor is simple. The notion of simplicity that we adopt is based on the \emph{(Kruskal) rank} of the tensor, which much like the matrix rank is of fundamental importance - tensors of lower rank have fewer constituent components and are hence simple. For example, video sequences are naturally modeled as tensors, and these third order tensors  have low rank as a result of homogeneous variations in the scene \cite{wang2014low}. Unlike the matrix case, however, computational tasks related to the tensor rank such as spectral decompositions, rank computation, and regularization are fraught with computational intractability \cite{tensorhard,tensorKolda} in the worst case.
%such as video data, graphical models with latent variables \cite{tensorcomp}, audio classification \cite{audio}, psychometrics \cite{psycho}, and neuroscience \cite{neuro}.

We focus on linear inverse problems involving tensors. Linear measurements of an unknown tensor $\vX$ are specified by $y = \mathcal{L}(\vX)$ where $\mathcal{L}$ is a linear operator and $y \in \R^{m}$. Here the quantity $m$ refers to the number of measurements, and the minimum number of measurements \footnote{We use the terms measurements and samples interchangeably.} $m$ required to reliably recover $\vX$ (called the \emph{sample complexity}) is of interest. While in general, such problems are ill-posed and unsolvable when $m$ is smaller than the dimensionality of $\vX$, the situation is more interesting when the underlying signal (tensor) is structured, and the sensing mechanism $\mathcal{L}(\cdot)$ is able to exploit this structure. For instance, similar ill-posed problems are solvable, even if $m$  is substantially lower than the ambient dimension, when the underlying signal is a sparse vector, or a low-rank matrix, provided that $\mathcal{L} (\cdot)$ has appropriate properties. 

We focus for the most part on tensors of order $3$, and show later that all our results extend to the higher order case in a straightforward way. We introduce a class of measurement operators known as \emph{separable measurements}, and present an algorithm for low-rank tensor recovery for the same. We  focus on two specific measurement mechanisms that are special cases of separable mechanisms:
\begin{itemize} 
\item  \emph{Separable random projections:} For tensors of order $3$, we consider observations where the $i^{th}$ measurement is of the form $\mathcal{L}_i(\vX):=\langle a\otimes A_i, \vX \rangle$, where $a$ is a random unit vector, $A_i$ is a random matrix, and $\otimes$ represents an outer product of the two. For higher order tensors, the measurements are defined in an analogous manner. Here $\langle \cdot, \cdot \rangle$ is the tensor inner product (to be made clear in the sequel).
\item \emph{Completion:} The measurements here are simply a subset of the entries of the true tensor. The entries need to be restricted to merely four slices of the tensor, and can be random within these slices. 
%Our algorithm breaks down the task of learning the entire tensor to that of learning just four slices of the tensor - a set of matrix completion problems. We show that one needs only sufficiently many entries in four slices - two across two dimensions - to be able to complete these four slices, and this is sufficient information to learn the entire low-rank tensor.
\end{itemize} 
For both the random projection and completion settings, we analyze the performance of our algorithm and prove sample complexity bounds.

The random sampling mechanisms mentioned above are of relevance in practical applications. For instance, the Gaussian random projection mechanism described above is a natural candidate for compressively sampling video and multi-dimensional imaging data. For applications where such data is ``simple'' (in the sense of low rank), the Gaussian sensing mechanism may be a natural means of compressive encoding.

The completion framework is especially relevant to machine learning applications. For instance, it is useful in the context of multi-task learning \cite{icml2013_romera-paredes13}, where each individual of a collection of inter-related tasks corresponds to matrix completion. Consider the tasks of predicting ratings assigned by users for different clothing items, this is naturally modeled as a matrix completion problem \cite{candesmatcomp}. Similarly, the task of predicting ratings assigned by the same set of users to accessories is another matrix completion problem. The multi-task of jointly predicting the ratings assigned by the users to baskets of items consisting of both clothing items and accessories is a tensor completion problem. 

Another application of tensor completion is that of extending the matrix completion framework for contextual recommendation systems. In such a setup, one is given a \emph{rating matrix} that is indexed by users and items, and the entries correspond to the ratings given by different users to different items. Each user provides ratings for only a fraction of the items (these constitute the sensing operator $\mathcal{L}\left( \cdot \right))$, and one wishes to infer the ratings for all the others. Assuming that such a rating matrix is low rank is equivalent to assuming the presence of a small number of latent variables that drive the rating process. An interesting twist to this setup which requires a tensor based approach is \emph{contextual} recommendation - i.e. where different users provide ratings in different contexts (e.g., location, time, activity). Such a setting is naturally modeled via tensors; the three modes of the tensor are indexed by users, items, and contexts. The underlying tensor may be assumed to be low rank to model the small number of latent variables that influence the rating process.  In this setting, our approach would need a few samples about users making decisions in two different contexts (this corresponds to two slices of the tensor along the third mode), and enough information about two different users providing ratings in a variety of different contexts (these are two slices along the first mode). Once the completion problem restricted to these slices is solved, one can complete the entire tensor by performing simple linear algebraic manipulations.

Of particular note concerning our algorithm and the performance guarantees are the following:
\begin{itemize}
%\item We consider \emph{low rank} sensing operators for tensor recovery. Our operators are very efficient from a storage point of view. Indeed, to sense an order $3$ tensor of size $n \times n \times n$, we require $O(n^2 + n)$ space compared to $O(n^3)$ for traditional sensing methods. 
\item \textbf{Sample complexity: }
In the absence of noise, our algorithm, named T-ReCs (Tensor Recovery via Contractions), provably and exactly recovers the true tensor and achieves an \emph{order-optimal} sample complexity for exact recovery of the underlying tensor in the context of random sensing, and order optimal modulo logarithmic factors in the context of tensor completion. Specifically, for a third order tensor of rank $r$ and largest dimension $n$, the achieved sample complexity is $O(nr)$ for recovery from separable random projections, and $O(nr \log^2n)$ for tensor completion. (These correspond to Theorems \ref{thm:exact_rec_sense} and \ref{thm:exact_rec_comp} respectively.) More generally, for order $K$ tensors the corresponding sample complexities are $O(Knr)$ and $O(Knr \log^2n)$ respectively (Theorems \ref{thm:exact_rec_high} and \ref{thm:exact_rec_highcomp}). 
%
%
%In the context of random sensing, our algorithm, named  in the absence of noise. We believe this to be the first result of this flavor. For a (third order) tensor of rank $r$ and largest dimension $n$, the achieved sample complexity is $O(nr)$, commensurate with the number of degrees of freedom in the model and hence optimal in its dependence on the rank as well as the dimension. More generally, for tensors of order $K$, dimension $n$, and rank $r$, the achieved sample complexity is $O(Knr)$.
%
%
%Similarly, for the completion problem, T-ReCs provably and exactly recovers the tensor
%
%
%In context of the completion problem, our algorithm provably and exactly recovers the true tensor and also achieves an almost order-optimal sample complexity for exact recovery of the underlying tensor in the absence of noise. For a (third order) tensor of rank $r$ and largest dimension $n$, the achieved sample complexity is $O(nr \log^2 n)$, which is order optimal (modulo logarithmic factors) in its dependence on the rank and the dimension. More generally, for tensors of order $K$, dimension $n$, and rank $r$, the achieved sample complexity is $O(Knr \log^2 n)$.

\item \textbf{Factorization:} Equally important is the fact that our method recovers a minimal rank factorization in addition to the unknown tensor. This is of importance in applications such as dimension reduction and also latent variable models \cite{animatensor} involving tensors where the factorization itself holds meaningful interpretational value.

\item \textbf{Absence of strong assumptions:} Unlike some prior art, our analysis relies only on relatively weak assumptions - namely that the rank of the tensor be smaller than the (smallest) dimension, that the factors in the rank decomposition be linearly independent, non-degenerate, and (for the case of completion) other standard assumptions such as incoherence between the factors and the sampling operator. We do not, for instance, require orthogonality-type assumptions of the said factors, as is the case in \cite{animatensor,Oh_tensor}.
 
\item \textbf{Computational efficiency:} Computationally, our algorithm essentially reduces to linear algebraic operations and the solution of matrix nuclear norm (convex) optimization sub-problems, and is hence extremely tractable. Furthermore, our nuclear norm minimization methods deal with matrices that are potentially much smaller, up to factors of $n$, than competing methods that ``matricize" the tensor via unfolding \cite{squaredeal, tomioka}. In addition to recovering the true underlying tensor, it also produces its unique rank decomposition.

\item \textbf{Simplicity:} Our algorithm is conceptually simple - both to implement as well as to analyze. Indeed the algorithm and its analysis follow in a transparent manner from Leurgans' algorithm (a simple linear algebraic approach for tensor decomposition) and standard results for low-rank matrix recovery and completion. We find this intriguing, especially considering the ``hardness'' of most tensor problems \cite{tensorhard,tensorKolda}. Recent work in the area of tensor learning has focused on novel regularization schemes and algorithms for learning low rank tensors; the proposed approach potentially obviates the need for developing these in the context of separable measurements.
\end{itemize}

The fundamental insight in this work is that while solving the tensor recovery problem directly may seem challenging (for example we do not know of natural tractable extensions of the ``nuclear norm'' for tensors), very important information is encoded in a two-dimensional matrix ``sketch'' of the tensor which we call a contraction. (This idea seems to first appear in \cite{unique1}, and is expanded upon in \cite{Moitra_tensor,Vempala_tensor}  in the context of tensor decomposition.) These sketches are formed by taking linear combinations of two-dimensional slices of the underlying tensor - indeed the  slices themselves may be viewed as ``extremal'' contractions. For the Gaussian random projections case, the contractions will be random linear combinations of slices, whereas for the completion setting the contractions we work with will be the slices themselves, randomly subsampled. Our method focuses on recovering these contractions efficiently (using matrix nuclear norm regularization) as a first step, followed by additional processing to recover the true tensor.

\subsection{Related Work and Key Differences}
With a view to computational tractability, the notion of \emph{Tucker rank} of a tensor has been explored; this involves \emph{matricizations} along different modes of the tensor and the ranks of the associated matrices. Based on the idea of Tucker rank, Tomioka et al. \cite{tomioka} have proposed and analyzed a nuclear norm heuristic for tensor completion, thereby bringing tools from matrix completion \cite{candesmatcomp} to bear for the tensor case. Mu et al. \cite{squaredeal}, have extended this idea further by studying reshaped versions of tensor matricizations. However, to date, the sample complexity associated to matrix-based regularization seem to be orders far from the anticipated sample complexity (for example based on a count of the degrees of freedom in the problem) \cite{squaredeal}. In this paper we resolve this conundrum by providing an efficient algorithm that provably enjoys order optimal sample complexity in the order, dimension, and rank of the tensor. 

%Alternating mimimization approaches for recovering tensor factors was proposed in \cite{Oh_tensor}. Again, the sample complexity bounds do not scale optimally with the dimensions or the rank. Unlike alternating minimization schemes that rely on a careful initialization of the method, our method directly solves convex optimization programs. 

In contrast to the matricization approach, alternative approaches for tensor completion with provable guarantees have appeared in the literature. In the restricted setting when the tensor has a symmetric factorization \cite{MOITRA} (in contrast we are able to work in the general non-symmetric setting), the authors propose employing the Lasserre hierarchy via a semidefinite programming based approach. Unfortunately, the method proposed in \cite{MOITRA} is not scalable - it requires solving optimization problems at the $6^{th}$ level of the Lasserre hierarchy which makes solving even moderate-sized problems numerically impractical as the resulting semidefinite programs grow rapidly with the dimension. Furthermore, the guarantees provided in \cite{MOITRA} are of a different flavor - they provide error bounds in the noisy setting, whereas we provide exact recovery results in the noiseless setting. Alternate methods based on thresholding in the noisy setting have also been studied in \cite{Aswani}.
An alternating minimization approach for tensor completion was proposed in \cite{Oh_tensor}. Their approach relies on the restrictive assumptions also - that the underlying tensor be symmetric and orthogonally decomposable (we make no such assumptions), and neither do the sample complexity bounds scale optimally with the dimensions or the rank. Unlike alternating minimization schemes that are efficient but rely on careful initializations, our method directly solves convex optimization programs followed by linear algebraic manipulations. Also relevant is \cite{yuan}, where the authors propose solving tensor completion using the tensor nuclear norm regularizer; this approach is not known to be computationally tractable (no polynomial time algorithm is known for minimizing the tensor nuclear norm) and the guarantees they obtain do not scale optimally with the dimension and rank. Finally a method based on the tubal rank and $t$-SVD of a tensor \cite{Aeron} has also recently been proposed, however the sample complexity does not scale optimally. As a final point of contrast to the aforementioned work, our method is also conceptually very simple - both to implement and analyze. 

In Table \ref{table:comparisons}, we provide a brief comparison of the relevant approaches, their sample complexities in both the third order and higher order settings as well as a few key features of each approach.

\small
\begin{table}[H] 
%\caption {Table Title} 
\label{table:comparisons} 
\begin{center}
%\begin{tabular}{ p{1.7cm} p{4.25cm}  p{4.25cm}  p{5cm} } 
\begin{tabular}{ccc p{4.5cm}}
\hline
 \textbf{Reference} & \textbf{Sample Complexity}  & \textbf{Sample Complexity} & \textbf{Key Features} \\ 
 &\textbf{($3^{rd}$ order)} &  \textbf{($K$th order)} & \\
 \hline \\
 \cite{tomioka} & $O(rn^2)$ & $O(rn^{K-1})$ & Tucker rank, tensor unfolding \\
 &&& \\
 \cite{squaredeal} & $O(rn^2)$ & $O(rn^{\lfloor{\frac{K}{2}}\rfloor})$ & Tucker rank, tensor unfolding \\
 &&& \\
 \cite{Oh_tensor} & $O(r^5n^{\frac{3}{2}}\log^{5}n)$ & - & Kruskal rank, alternating minimization, orthogonally decomposable tensors, symmetric setting, completion only.  \\
 &&& \\
 \cite{Aeron} & $O(rn^2\log n)$ & - & Tensor tubal rank, completion only \\
 &&& \\
 \cite{yuan} & $O(r^{\frac{1}{2}}(n \log n)^{\frac{3}{2}})$ & $O(n^{\frac{K}{2}}\text{polylog}(n))$ & Kruskal rank, Exact tensor nuclear norm minimization, computationally intractable, completion only. \\
 &&&\\
 {Our Method}
  & $O(nr)$ (random projection)  & $O(Knr)$ (random projection) & Kruskal rank, separable \\
 & $O(nr\log^2 n)$ (completion) & $O(Knr\log^2 n)$ (completion) & measurements, Leurgans' algorithm\\
 &&& \\
  \hline
 \label{table:comparisons}
\end{tabular}
 \caption{Table comparing sample complexities of various approaches.}
\end{center}
\end{table}

\normalsize

%
%
%
%ORGANIZATION OF PAPER:
%
%Things to discuss:
%\begin{itemize}
%\item different sensing mechansims and their need, e.g. compression
%\item getting optimal sample complexity - compare to other in literature
%\item Some literature search - Tomioka, Sewoong, Moitra, etc.
%\end{itemize}

The rest of the paper is organized as follows: in Section \ref{sec:prelim}, we introduce the problem setup and describe the approach and result in the most general setting. We also describe Leurgans' algorithm, an efficient linear algebraic algorithm for tensor decomposition, which our results build upon. In Section \ref{sec:mainresults} we specialize our results for both the random projections and the tensor completion cases. We extend these results and our algorithm to higher order tensors in Section \ref{sec:higher_order}. We perform experiments that validate our theoretical results in Section \ref{sec:exp}. In Section \ref{sec:conclusion}, we conclude the paper and outline future directions.
%Throughout the paper, many of the proofs are deferred to the supplementary material in the interests of space.

% before concluding our paper in Section \ref{sec:conc}.   

%%%%%%%%%%%%%%%%%%%%%%%%%%%%%%%%%%%%%%%%%%%%%%%%
%%%%%%%%%%%%%%%%%%%%%%%%%%%%%%%%%%%%%%%%%%%%%%%%
%%%%%%%%%%%%%%%%%%%%%%%%%%%%%%%%%%%%%%%%%%%%%%%%
%%%%%%%%%%%%%%%%%%%%%%%%%%%%%%%%%%%%%%%%%%%%%%%%
%%%%%%%%%%%%%%%%%%%%%%%%%%%%%%%%%%%%%%%%%%%%%%%%
%%%%%%%%%%%%%%%%%%%%%%%%%%%%%%%%%%%%%%%%%%%%%%%%
%%%%%%%%%%%%%%%%%%%%%%%%%%%%%%%%%%%%%%%%%%%%%%%%
%%%%%%%%%%%%%%%%%%%%%%%%%%%%%%%%%%%%%%%%%%%%%%%%
%%%%%%%%%%%%%%%%%%%%%%%%%%%%%%%%%%%%%%%%%%%%%%%%
%%%%%%%%%%%%%%%%%%%%%%%%%%%%%%%%%%%%%%%%%%%%%%%%

%%%%%%%%%%%%%%%%%%%%%%%%%%%%%%%%%%%%%%%%%%%%%%%%
%%%%%%%%%%%%%%%%%%%%%%%%%%%%%%%%%%%%%%%%%%%%%%%%
%%%%%%%%%%%%%%%%%%%%%%%%%%%%%%%%%%%%%%%%%%%%%%%%
%%%%%%%%%%%%%%%%%%%%%%%%%%%%%%%%%%%%%%%%%%%%%%%%
%%%%%%%%%%%%%%%%%%%%%%%%%%%%%%%%%%%%%%%%%%%%%%%%
%%%%%%%%%%%%%%%%%%%%%%%%%%%%%%%%%%%%%%%%%%%%%%%%
%%%%%%%%%%%%%%%%%%%%%%%%%%%%%%%%%%%%%%%%%%%%%%%%
%%%%%%%%%%%%%%%%%%%%%%%%%%%%%%%%%%%%%%%%%%%%%%%%
%%%%%%%%%%%%%%%%%%%%%%%%%%%%%%%%%%%%%%%%%%%%%%%%
%%%%%%%%%%%%%%%%%%%%%%%%%%%%%%%%%%%%%%%%%%%%%%%%

\section{Approach and Basic Results}
\label{sec:prelim}
\vspace{-2mm}
In this paper, vectors are denoted using lower case characters (e.g. $x, y, a, b,$ etc.), matrices by upper-case characters (e.g. $X, Y,$ etc,) and tensors by upper-case bold characters (e.g. $\vX, \vT, \vA$ etc.). Given two third order tensors $\vA, \vB$, their inner product is defined as:
\[
\langle \vA, \vB \rangle = \sum_{i, j, k}  \vA_{ijk} \vB_{ijk}.
\]
The Euclidean norm of a tensor $\vA$ is generated by this inner product, and is a straightforward extension of the matrix Frobenius norm:
\[
\| \vA \|_{F}^2:= \langle \vA, \vA \rangle. 
\]

We will work with tensors of third order (representationally to be thought of as three-way arrays), and the term mode refers to one of the axes of the tensor. A slice of a tensor refers to a two dimensional matrix generated from the tensor by varying indices along two modes while keeping the third mode fixed.
For a tensor $\vX$ we will refer to the indices of the $i^{th}$ mode-$1$ slice (i.e., the slice corresponding to the indices $\left\{i \right\} \times [n_2] \times [n_3]$) by $S_i^{(1)}$, where $[n_2] = \{1, 2, \ldots, n_2\}$ and $[n_3]$ is defined similarly. We denote the matrix corresponding to $S_i^{(1)}$ by $X^{1}_i$. Similarly the indices of the $k^{th}$ mode-$3$ slice will be denoted by $S_k^{(3)}$ and the matrix by $X^{3}_k$.

%PROPERLY DEFINE MODES, SLICES, SUBSCRIPT NOTATION ETC.
%Given two (third order) tensors $\mA, ~\ \mB$, we define the inner product between them as 
%\[
%\langle \mA, \mB \rangle = \sum_{i,j,k} \mA_{ijk} \mB_{ijk}
%\]
%A similar definition holds for higher order tensors as well. 

Given a tensor of interest $\vX$, consider its decomposition into rank one tensors
\begin{equation} \label{eq:decomp0}
\vX=\sum_{i=1}^r u_i \otimes v_i \otimes w_i,
\end{equation}
where $\left\{ u_i\right\}_{i=1, \ldots, r} \subseteq \R^{n_1}$, $\left\{ v_i\right\}_{i=1, \ldots, r} \subseteq \R^{n_2}$, and $\left\{ w_i\right\}_{i=1, \ldots, r} \subseteq \R^{n_3}$. Here $\otimes$ denotes the tensor product, so that $\vX \in \R^{n_1 \times n_2 \times n_3}$ is a tensor of order $3$ and dimension $n_1\times n_2 \times n_3$. Without loss of generality, throughout this paper we assume that $n_1 \leq n_2 \leq n_3$. We will first present our results for third order tensors, and analogous results for higher orders follow in a transparent manner. We will be dealing with \emph{low-rank} tensors, i.e. those tensors with $r \leq n_1$. Tensors can have rank larger than the dimension, indeed $r \geq n_3$ is an interesting regime, but far more challenging and will not be dealt with here.

Kruskal's Theorem \cite{Kruskal1} guarantees that tensors satisfying Assumption \ref{assump:1} below have a unique minimal decomposition into rank one terms of the form \eqref{eq:decomp0}. The minimal number of terms is called the (Kruskal) rank\footnote{The Kruskal rank is also known as the CP rank in the literature.} of the tensor $\vX$.%\begin{assumption} \label{assump:1}
%The set $\left\{ u_i\right\}_{i=1, \ldots, r} \subseteq \R^{n_1}$ is a set of linearly independent vectors, the set $\left\{ v_i\right\}_{i=1, \ldots, r} \subseteq \R^{n_2}$ is a set of linearly independent vectors and the set $\left\{ w_i\right\}_{i=1, \ldots, r} \subseteq \R^{n_3}$ is a set of \emph{pairwise} independent vectors (i.e. every pair of vectors therein is independent).
%\end{assumption}
\begin{assumption} \label{assump:1}
The sets $\left\{ u_i\right\}_{i=1, \ldots, r} \subseteq \R^{n_1}$ and $\left\{ v_i\right\}_{i=1, \ldots, r} \subseteq \R^{n_2}$ are sets of linearly independent vectors and the set $\left\{ w_i\right\}_{i=1, \ldots, r} \subseteq \R^{n_3}$ is a set of \emph{pairwise} independent vectors
\end{assumption}
While rank decomposition of tensors in the worst case is known to be computationally intractable \cite{tensorhard}, it is known that the  (mild) assumption stated in Assumption \ref{assump:1} above suffices for an algorithm known as Leurgans' algorithm \cite{unique1,Moitra_tensor} to correctly identify the factors in this unique decomposition.
In this paper, we will work with the following, somewhat stronger assumption:

\begin{assumption} \label{assump:2}
The sets $\left\{ u_i\right\}_{i=1, \ldots, r} \subseteq \R^{n_1}$, $\left\{ v_i\right\}_{i=1, \ldots, r} \subseteq \R^{n_2}$, and $\left\{ w_i\right\}_{i=1, \ldots, r} \subseteq \R^{n_3}$ are sets of linearly independent vectors. 
%We also assume that these vectors, when normalized to unit Euclidean norm are all distinct, i.e. 
%$$\frac{u_i}{\|u_i\|}\neq \frac{v_j}{\|v_j\|} \neq \frac{w_k}{\|w_k\|}$$ for all $i,j,k$.
\end{assumption}

%%%%%%%%%%%%%%%%%%%%%%%%%%%%%%%%%%%%%%%%%%
%%%%%%%%%%%%%%%%%%%%%%%%%%%%%%%%%%%%%%%%%%
%%%%%%%%%%%%%%%%%%%%%%%%%%%%%%%%%%%%%%%%%%
%%%%%%%%%%%%%%%%%%%%%%%%%%%%%%%%%%%%%%%%%%
%%%%%%%%%%%%%%%%%%%%%%%%%%%%%%%%%%%%%%%%%%
%%%%%%%%%%%%%%%%%%%%%%%%%%%%%%%%%%%%%%%%%%
%%%%%%%%%%%%%%%%%%%%%%%%%%%%%%%%%%%%%%%%%%
%%%%%%%%%%%%%%%%%%%%%%%%%%%%%%%%%%%%%%%%%%
%%%%%%%%%%%%%%%%%%%%%%%%%%%%%%%%%%%%%%%%%%
%%%%%%%%%%%%%%%%%%%%%%%%%%%%%%%%%%%%%%%%%%

%%%%%%%%%%%%%%%%%%%%%%%%%%%%%%%%%%%%%%%%%%
%%%%%%%%%%%%%%%%%%%%%%%%%%%%%%%%%%%%%%%%%%
%%%%%%%%%%%%%%%%%%%%%%%%%%%%%%%%%%%%%%%%%%
%%%%%%%%%%%%%%%%%%%%%%%%%%%%%%%%%%%%%%%%%%
%%%%%%%%%%%%%%%%%%%%%%%%%%%%%%%%%%%%%%%%%%
%%%%%%%%%%%%%%%%%%%%%%%%%%%%%%%%%%%%%%%%%%
%%%%%%%%%%%%%%%%%%%%%%%%%%%%%%%%%%%%%%%%%%
%%%%%%%%%%%%%%%%%%%%%%%%%%%%%%%%%%%%%%%%%%
%%%%%%%%%%%%%%%%%%%%%%%%%%%%%%%%%%%%%%%%%%
%%%%%%%%%%%%%%%%%%%%%%%%%%%%%%%%%%%%%%%%%%

\subsection{Separable Measurement Mechanisms} \label{sec:separable}
As indicated above in the preceding discussions, we are interested in tensor linear inverse problems where, given measurements of the form $y_i= \mathcal{L}_i\left( \vX \right) ~\ i = 1, 2, \cdots, m$,  we recover the unknown tensor $\vX$. We focus on a class of measurement mechanisms $\mathcal{L}\left( \cdot \right)$ which have a special property which we call \emph{separability}. We define the notion of separable measurements formally:
\begin{definition} \label{def:separable}
Consider a linear operator $\mathcal{L}: \R^{n_1 \times n_2 \times n_3} \rightarrow \R^n$. We say that $\mathcal{L}$ is separable with respect to the third mode if there exist $w \in \R^{n_3}$ and a linear operator $\mathcal{T}: \R^{n_1 \times n_2 } \rightarrow \R^n$, such that for every $\vX \in \R^{n_1 \times n_2 \times n_3}$: 
$$
\mathcal{L}\left( \vX \right) = \sum_{i=1}^{n_3} w_i \mathcal{T} \left( X^3_i \right).
$$
\end{definition}
This definition extends in a natural way for separability of operators with respect to the second and first modes. In words, separability means that the effect of the linear operator $\mathcal{L}\left( \cdot \right)$ on a tensor can be decomposed into the (weighted) sum of actions of a \emph{single} linear operator $\mathcal{T}\left( \cdot \right)$ acting on \emph{slices} of the tensor along a particular mode.

In several applications involving inverse problems, the design of appropriate measurement mechanisms is itself of interest. Indeed sensing methods that lend themselves to recovery from a small number of samples via efficient computational techniques has been intensely studied in the signal processing, compressed sensing, and machine learning literature \cite{Candes:2006eq,CandesTao2, recht2010guaranteed, Ben, CovSketch, Tang:2013fo}. In the context of tensors, we argue, separability of the measurement operator is a desirable property for precisely these reasons; because it lends itself to recovery up to almost optimal sample complexity via scalable computational methods (See for example \cite{kueng2014low} for rank one measurement operators in the matrix case). We now describe a few interesting measurement mechanisms that are separable.

\begin{enumerate}
\item \textbf{Separable random  projections:} Given a matrix $M \in \R^{n_1 \times n_2}$ and vector $v \in \R^{n_3}$, we define the following two notions of ``outer products'' of $M$ and $v$: 
\begin{align*}
\left[ M\otimes v \right]_{ijk}:= M_{ij}v_k \qquad \left[ v \otimes M \right]_{ijk}:= v_i M_{jk}.
\end{align*}
Hence, the $k^{th}$ mode $3$ slice of the tensor $M\otimes v $ is the matrix $v_k M$. Similarly, the $i^{th}$ mode $1$ slice of the tensor $v\otimes M $ is the matrix $v_i M$.

A typical random separable projection is of the form:
\begin{align} \label{eq:rand_proj}
\mathcal{L}\left( \vX \right) = \left[ \begin{array}{c} \langle A_1 \otimes a, \vX \rangle \\
\vdots \\
\langle A_m \otimes a, \vX \rangle
\end{array}\right]
\end{align}
where $A_i \in \R^{n_1 \times n_2}$ is a random matrix drawn from a suitable ensemble such as the Gaussian ensemble with each entry drawn independently and identically from $\mathcal{N}(0,1)$, and $a \in \R^{n_3}$ is also a random vector, for instance distributed uniformly on the unit sphere in $n_3$ dimensions (i.e. with each entry drawn independently and identically from $\mathcal{N}(0,1)$ and then suitably normalized). 

To see that such measurements are separable, note that:
\begin{align*}
\langle A_i \otimes a, \vX \rangle =  \sum_{k=1}^{n_3} a_k \langle A_i, X_k^3 \rangle, 
\end{align*}
so that the operator $\mathcal{T}\left( \cdot \right)$ from Definition \ref{def:separable} in this case is simply given by:
$$
\mathcal{T}\left( X \right) = \left[ \begin{array}{c} \langle A_1, X \rangle \\
\vdots \\
\langle A_m, X \rangle
\end{array}\right].
$$

Random projections are of basic interest in signal processing, and have played a key role in the development of sparse recovery and low rank matrix recovery literature \cite{recht2010guaranteed,CandesTao2}. From an application perspective they are relevant because they provide a method of compressive and lossless coding of ``simple signals'' such as sparse vectors \cite{CandesTao2} and low rank matrices \cite{recht2010guaranteed}. In subsequent sections we will establish that separable random projections share this desirable feature for low-rank tensors.

\item \textbf{Tensor completion:} In tensor completion, a subset of the entries of the tensor $\vX$ are revealed. 
Specifically, given a tensor $\vX $, a subset of the entries $\vX_{ijk}$ for $i,j,k \in \Omega$ are revealed for some index set $\Omega \subseteq [n_1] \times [n_2] \times [n_3]$ (we denote this by $\left( \vX\right) _{\Omega}$). Whether or not the measurements are separable depends upon the nature of the set $\Omega$.
For the $i^{th}$ mode-$1$ slice let us define 
\begin{align*}
\Omega^{(1)}_i:=\Omega \cap S^{(1)}_i \qquad m^{(1)}_i:=\left| \Omega \cap S_i^{(1)} \right|.
\end{align*}
Measurements derived from entries within a \emph{single slice} of the tensor are separable. This follows from the fact that for 
$\mathcal{L} \left( \vX \right):=\left( \vX \right)_{\Omega^{(1)}_{i}}$, we have:
$$
\mathcal{L} \left( \vX \right) = \sum_{j=1}^{n_1} \left( \delta_{i} \right)_j \mathcal{M}_{\Omega^{(1)}_{i}} \left( X^{(1)}_j\right) 
$$
where $ \delta_{i} \in \R^{n_1}$ is a vector with a one is the $i$ index and zero otherwise, and $ \mathcal{M}_{\Omega}$ is the operator that acts on a matrix $X$, extracts the indices corresponding to the index $\Omega$, and returns the resulting vector. Comparing to Definition \ref{def:separable}, we have $w=\delta_{i}$ and $\mathcal{T}=\mathcal{M}_{\Omega^{(1)}_{i}}$.
As a trivial extension, measurements obtained from parallel slices where the index set restricted to these slices is identical are also separable.

Analogous to matrix completion, tensor completion is an important problem due to its applications to machine learning; the problems of multi-task learning and contextual recommendation are both naturally modeled in this framework as described in Section \ref{sec:intro}. 

\item \textbf{Rank one projections} 
Another separable sensing mechanism of interest is via rank-one projections of the tensor of interest. Specifically, measurements of the form:
$$
\mathcal{L}\left( \vX \right) = \left[ \begin{array}{c}
\langle a_1 \otimes b_1 \otimes c, \vX \rangle \\
\vdots \\
\langle a_m \otimes b_m \otimes c, \vX \rangle
\end{array} \right]
$$
are also separable. Mechanisms of this form have recently gained interest in the context of low rank (indeed rank-one) matrices due to their appearance in the context of phase retrieval problems \cite{PhaseLift} and statistical estimation \cite{kueng2014low,cai2015}. We anticipate that studying rank one projections in the context of tensors will give rise to interesting applications in a similar spirit.

\item \textbf{Separable sketching}
The notion of covariance sketching (and more generally, matrix sketching) \cite{CovSketch} allows for the possibility of compressively acquiring a matrix $X$ via measurements $Y=AXB^T$, where $A \in \R^{m_1 \times p}$ and $B \in \R^{m_2 \times p}$, $X \in \R^p$, and $m_1, m_2 <p$. The problem of recovering $X$ from such measurements is of interest in various settings such as when one is interested in recovering a covariance matrix from compressed sample paths, and graph compression \cite{CovSketch}. In a similar spirit, we introduce the notion of separable sketching of tensors defined via: 
$$
Y_{qs}=\mathcal{L}\left(\vX\right)=\sum_{i=1}^{n_1} \sum_{j=1}^{n_2} \sum_{k=1}^{n_3}A_{qi} B_{sj} c_{k}\vX_{ijk}.
$$
In the above $A \in \R^{m_1 \times n_1}$, $B \in \R^{m_2 \times n_2}$, $c \in \R^{n_3}$, and $Y \in \R^{m_1 \times m_2}$. Note that $\mathcal{T}\left( Z\right) = AZB^T$, i.e. precisely a matrix sketch of tensor slices. The problem of recovering $\vX$ from $Y$ is thus a natural extension of matrix sketching to tensors.

Finally, we note that while a variety of separable sensing mechanisms are proposed above, many sensing mechanisms of interest are not separable. For instance, a measurement of the form $\mathcal{L}(\vX)=\langle \vA, \vX \rangle$ where $\vA$ is a full rank tensor is \emph{not} separable. Similarly, completion problems where entries of the tensor are revealed randomly and uniformly throughout the tensor (as apposed to from a single slice) are also not separable (although they may be thought of as a union of separable measurements). In Section \ref{sec:third_order}, we will provide sample complexity bounds for exact recovery for the first two aforementioned measurement mechanisms (i.e. random projections and tensor completion); the arguments extend in a natural manner to other separable sensing mechanisms.
\end{enumerate}

\subsubsection{Diversity in the  Measurement Set} \label{sec:diversity}
In order to recover the low rank tensor from a few measurements using our algorithm, we need the set of measurements to be a \emph{union} of separable measurements which satisfy the following:
\begin{enumerate}
\item \textbf{Diversity across modes:} Measurements of the form \eqref{eq:rand_proj} are separable with respect to the third mode. For the third order case, we also need an additional set of measurements separable with respect to the first mode \footnote{Any two modes suffice. In this paper we will focus on separability w.r.t the first and third modes.}. This extends naturally also to the higher order case.
\item \textbf{Diversity across separable weights:} Recalling the notion of weight vectors, $w$, from Definition: \ref{def:separable}, we require that for both modes $1$ and $3$, each mode has two distinct sets of separable measurements with distinct weight vectors.
\end{enumerate}

To make the second point more precise later, we introduce the formal notation we will use in the rest of the paper for the measurement operators:
$$
y_k^{(i)}=\mathcal{L}^{(i)}_k\left( \vX \right) =  \sum_{j=1}^{n_3} \left(w_k^{(i)}\right)_j \mathcal{T}^{(i)}_k \left( X^3_j \right)
$$

In the above, the index $i \in \left\{1, 3\right\}$ refers to the mode with respect to which that measurement is separable. For each mode, we have two distinct sets of measurements corresponding to two different weight vectors $w_k^{(i)}$, with $k \in \left\{1,2 \right\}$. For each $k$ and $i$, we may have potentially different operators $\mathcal{T}_k^{(i)}$ (though they need not be different). To simplify notation, we will subsequently assume that $\mathcal{T}_1^{(i)} = \mathcal{T}_2^{(i)} = \mathcal{T}^{(i)}$. Collectively, all these measurements will be denoted by:
$$
y = \mathcal{L}  \left( \vX \right), 
$$
where it is understood that $y$ is a concatenation of the vectors $y^{(i)}_k$ and similarly $\mathcal{L}\left( \cdot \right)$ is a concatenation of  $\mathcal{L}^{(i)}_k\left( \cdot \right)$.
We will see in the subsequent sections that when we have diverse measurements across different modes and different weight vectors, and when the $\mathcal{T}^{(i)}$ are chosen suitably, one can efficiently recover an unknown tensor from an (almost) optimal number of measurements of the form $y = \mathcal{L}\left( \vX \right)$.

\subsection{Tensor Contractions}
A basic ingredient in our approach is the notion of a tensor contraction. This notion will allow us to form a bridge between inverse problems involving tensors and inverse problems involving matrices, thereby allowing us to use matrix-based techniques to solve tensor inverse problems.

For a tensor $\vX$, we define its mode-$3$ \emph{contraction} with respect to a contraction vector $a \in \R^{n_3}$, denoted by $X^{3}_a\in \R^{n_1 \times n_2}$, as the following matrix:
\begin{equation} \label{eq:contr_def}
\left[ X^{3}_a \right]_{ij} = \sum_{k=1}^{n_3} \vX_{ijk} a_k,
\end{equation}
so that the resulting matrix is a weighted sum of the mode-$3$ slices of the tensor $\vX$.
We similarly define the mode-$1$ contraction with respect to a vector $c \in \R^{n_{1}}$ as
\begin{equation}
\left[ X^{1}_c \right]_{jk} = \sum_{k=1}^{n_1} \vX_{ijk} c_i,
\end{equation}
Note that when $a=e_k$, a standard unit vector, $X_a^3=X_k^{3}$, i.e. a tensor slice. We will primarily be interested in two notions of contraction in this paper:
\begin{itemize}
\item \emph{Random Contractions},  where $a$ is a random vector distributed uniformly on the unit sphere. These will play a role in our approach for recovery from random projections.
\item \emph{Coordinate Contractions}, where $a$ is a canonical basis vector, so that the resulting contractions is a tensor slice. These will play a role in our tensor completion approach. 
\end{itemize}

We now state a basic result concerning tensor contractions. %\textcolor{red}{Should we move these 2 Lemmas to the appendix?}

\begin{lemma} \label{lemma:contr_rank}
Let $\vX \in \R^{n_1 \times n_2 \times n_3}$, with $n_1 \leq n_2 \leq n_3$ be a tensor of rank $r \leq n_1$. Then the rank of $X^{3}_a$ is at most $r$. Similarly, if $r \leq \min \left\{ n_2, n_3 \right\}$ then the rank of $X^1_c$ is at most $r$. 
\end{lemma}
\begin{proof}
Consider a tensor $\vX = \sum_{i=1}^r u_i \otimes v_i \otimes w_i$.
The reader may verify in a straightforward manner that $X^{3}_a$ enjoys the decomposition:
\begin{equation} \label{eq:decomp}
X^{3}_a=\sum_{i=1}^r \langle w_i, a \rangle u_i v_i^{T}.
\end{equation}
The proof for the rank of $X^1_c$ is analogous.
\end{proof}

Note that while \eqref{eq:decomp} is a matrix decomposition of the contraction, it is not a singular value decomposition (the components need not be orthogonal, for instance). Indeed it does not seem ``canonical'' in any sense. Hence, given contractions, resolving the components is a non-trivial task. 

A particular form of degeneracy we will need to avoid is situations where $\langle w_i, a \rangle =0$ for \eqref{eq:decomp}.
It is interesting to examine this in the context of coordinate contractions, i.e. when
 $a=e_k$, we have $X^3_{e_k}=X^k_3$ (i.e. the $k^{th}$ mode $3$ slice), by Lemma \ref{lemma:contr_rank}, we see that the tensor slices are also of rank at most $r$. 
Applying $a=e_k$ in the decomposition \eqref{eq:decomp}, we see that if for some vector $w_i \in \R^{n_3}$ in the above decomposition we have that the $k^{th}$ component of $w_i$ (i.e. $(w_i)_{k}$) is zero then $\langle w_i, e_k \rangle =0$, and hence this component is missing in the decomposition of $X^{3}_k$. As a consequence the rank of $X^3_k$ drops, and in a sense information about the factors $u_k, v_k$ is ``lost'' from the contraction. We will want to avoid such situations and thus introduce the following definition:

\begin{definition} \label{def:degen}
Let $\vX= \sum_{i=1}^r u_i \otimes v_i \otimes w_i$. We say that the contraction $X_a^3$ is non-degenerate if $\langle w_i, a \rangle \neq 0$, for all $i=1, \ldots, r$.
\end{definition}
We will extend the terminology and say that the tensor $\vX$ is \emph{non-degenerate} at mode $3$ and component $k$ if the $k^{th}$ tensor slice is non-degenerate, i.e. component $k$ of the vectors $w_i$, $i=1, \ldots, r$ are all non-zero.
The above definition extends in a natural way to other modes and components. 
The non-degeneracy condition is trivially satisfied (almost surely) when:
\begin{enumerate}
\item The vector $a$ with respect to which the contraction is computed is suitably random, for instance random normal. In such situations, non-degeneracy holds almost surely.
\item When $a=e_k$ (i.e. the contraction is a slice), and the tensor factors are chosen from suitable random ensembles, e.g. when the low rank tensors are picked such that the rank one components $u_i, v_i, w_i$ are Gaussian random vectors, or random orthogonal vectors \footnote{the latter is known as random orthogonal model in the matrix completion literature \cite{Ben}}. 
\end{enumerate}

We will also need the following definition concerning the genericity of a pair of contractions: 
\begin{definition}
Given a tensor $\vX = \sum_{i=1}^r u_i \otimes v_i \otimes w_i$, a pair of contractions $X^3_a, X^3_b$ are \emph{pairwise generic} if the diagonal entries of the (diagonal)  $D_a D_b^{-1}$ are all distinct, where $D_a = \mathrm{diag} \left( \langle w_1, a \rangle, \ldots, \langle w_r, a \rangle \right)$, $D_b = \mathrm{diag} \left( \langle w_1, a \rangle, \ldots, \langle w_r, b\rangle \right)$. 
\end{definition}
%Note that since $X, Y$ are of rank $r$, the matrices $D_X, D_Y \in R^{r \times r}$ are invertible. When $X,Y$ are matrices obtained as random contractions of a low-rank tensor, they are pairwise generic.

\begin{remark}
We list two cases where pairwise genericity conditions hold in this paper.
\begin{enumerate}
\item In the context of random contractions, for instance when the contraction vectors $a, b$ are sampled uniformly and independently on the unit sphere. In this case pairwise genericity holds almost surely. 
%In this case the matrices $D_a = \text{diag} \left( \langle w_1, a \rangle, \ldots, \langle w_r, a \rangle \right)$, $D_b = \text{diag} \left( \langle w_1, a \rangle, \ldots, \langle w_r, b\rangle \right)$ are such that $D_aD_b^{-1}$ has distinct entries.
\item  In the context of tensor completion where $a=e_{k_{1}}, b=e_{k_{2}}$, the two diagonal matrices $D_a = \text{diag}\left( \left(w_1\right)_{k_1}, \ldots, \left(w_r\right)_{k_1} \right)$, and $D_b = \text{diag}\left( \left(w_1\right)_{k_2}, \ldots, \left(w_r\right)_{k_2} \right)$. Thus the pairwise genericity condition is a genericity requirement of the tensor factors themselves, namely that the ratios $ \frac{\left(w_i\right)_{k_1}}{\left(w_i\right)_{k_2}}$ all be distinct for $i=1, \ldots, r$. We will abuse terminology, and call such a tensor pairwise generic with respect to mode $3$ slices $k_1, k_2$. This form of genericity is easily seen to hold, for instance when the tensor factors are drawn from suitable random ensembles such as random normal and random uniformly distributed on the unit sphere.
\end{enumerate}
\end{remark}
The next lemma, a variation of which appears in \cite{Moitra_tensor,unique1} shows that when the underlying tensor is non-degenerate, it is possible to decompose a tensor from pairwise generic contractions.

\begin{lemma}\cite{Moitra_tensor,unique1} \label{lemma:fund_lemma}
Suppose we are given an order 3 tensor $\vX = \sum_{i = 1}^r u_i \otimes v_i \otimes w_i$ of size $n_1 \times n_2 \times n_3$ satisfying the conditions of Assumption \ref{assump:1}. Suppose the contractions $X_a^{3}$ and $X_b^3$ are non-degenerate, and consider the matrices $M_1$ and $M_2$ formed as:
\begin{align*}
M_1= X_a^3 (X_b^3)^\dagger \qquad M_2= (X_b^3)^\dagger X_a^3.
\end{align*}
Then the eigenvectors of $M_1$ (corresponding to the non-zero eigenvalues) are $\left\{ u_i \right\}_{i=1, \ldots, r}$, and the eigenvectors of $M_2^T$ are $\left\{ v_i \right\}_{i=1, \ldots, r}$.
\end{lemma}
\begin{proof}
Suppose we are given an order 3 tensor $\vX = \sum_{i = 1}^r u_i \otimes v_i \otimes w_i \in \R^{n_1 \times n_2 \times n_3}$.
From the definition of contraction \eqref{eq:contr_def}, it is straightforward to see that 
\[
X_a^3 = UD_aV^T ~\ D_a = \mbox{diag}(a^Tw_1, \ldots, a^Tw_r)
\]
\[
X_b^3 = UD_bV^T ~\ D_b = \mbox{diag}(b^Tw_1, \ldots, b^Tw_r).
\]
In the above decompositions, $U \in \R^{n_1 \times r}$, $V \in \R^{n_2 \times r}$, and the matrices $D_a, D_b\in \R^{r \times r}$ are diagonal and non-singular (since the contractions are non-degenerate).
Now, 
\begin{align}
\notag
M_1 &:= X_a^3 (X_b^3)^\dagger \\
\notag
&= UD_aV^T (V^{\dagger})^{T}D_b^{-1}U^{\dagger} \\
\label{udiag}
&= UD_aD_b^{-1}U^\dagger
\end{align}
and similarly we obtain
\begin{equation}
\label{vdiag}
M_2^T = VD_b^{-1}D_aV^\dagger.
\end{equation}
Since we have $M_1U=U D_aD_b^{-1} $ and $M_2^TV=VD_b^{-1} D_a$, it follows that the columns of $U$ and $V$ are eigenvectors of $M_1$ and $M_2^T$ respectively (with corresponding eigenvalues given by the diagonal matrices $D_aD_b^{-1}$ and $D_b^{-1}D_a$).
\end{proof}
\begin{remark}
Note that while the eigenvectors $\left\{u_i\right\}, \left\{v_j\right\}$ are thus determined, a source of ambiguity remains. For a fixed ordering of the $u_i$ one needs to determine the order in which the $v_j$ are to be arranged. This can be (generically) achieved by using the (common) eigenvalues of $M_1$ and $M_2$ for pairing.
If the contractions $X_a^3, X_b^3$ satisfy pairwise genericity, we see that the diagonal entries of the matrix $D_a D_b^{-1}$ are distinct. It then follows that the eigenvalues of $M_1, ~ M_2$ are distinct, and can be used to pair the columns of $U$ and $V$. 
\end{remark}

%For the second case we note that by the non-degeneracy assumption, the matrix 
%\begin{align*}
%D_a = \text{diag}\left( e_{k_1}^Tw_1, \ldots,  e_{k_1}^Tw_r \right) \qquad D_b = \text{diag}\left( e_{k_2}^Tw_1, \ldots,  e_{k_2}^Tw_r \right).
%\end{align*}
%Due to the non-degeneracy assumption, the diagonal entries of $D_a$ and $D_b$ are non-zero and distinct.  By a similar line of reasoning as above, it then follows that the eigenvectors of $M_1, ~ M_2^T$ corresponding to the non-zero eigenvalues are unique, and we can pair the columns of $U$ and $V$ based on their eigenvalues. 

\subsection{Leurgans' algorithm}
We now describe Leurgans' algorithm for tensor decomposition in Algorithm \ref{alg: Leurgans}. In the next section, we build on this algorithm to solve tensor inverse problems to obtain optimal sample complexity bounds. In words, Algorithm  \ref{alg: Leurgans} essentially turns a problem involving decomposition of tensors into that of decomposition of matrices. This is achieved by first computing mode $3$ contractions of the given tensor $\vX$ with respect to two non-degenerate and pairwise generic vectors $a, b$ (e.g. randomly uniformly distributed on the unit sphere). Given these contractions, one can compute matrices $M_1$ and $M_2$ as described in Lemma \ref{lemma:fund_lemma} whose eigenvectors turn out to be \emph{precisely} (up to scaling) the vectors $u_i$ and $v_i$ of the required decomposition. Finally the $w_i$ can be obtained by inverting an (overdetermined) system of linear equations, giving a unique and exact solution. 

%\textcolor{red}{GT: this part still talks about random $a$ and $b$, even though we also consider coordinates. Perhaps we should redefine non-degeneracy and just mention using non-degenerate $a$ and $b$.}

The correctness of the algorithm follows directly from Lemma \ref{lemma:fund_lemma}. 

%\textcolor{red}{ NR: this is repeated: The key observation is that when Assumption \ref{assump:1} is satisfied, the eigenvectors of $M_1$ and $M_2$ respectively are precisely the factors (upto rescaling) $u_i$ and $v_i$ of the underlying tensor. Once these are recovered, the factors $w_i$ can be correctly recovered by solving an overdetermined linear system of equations with a unique and exact solution.}
In this paper, we extend this idea to solving ill-posed linear inverse problems of tensors. The key idea is that since the contractions preserve information about the tensor factors, we focus on recovering the contractions first. Once those are recovered, we simply need to compute eigendecompositions to recover the factors themselves.
\begin{algorithm}[!ht]
   \caption{Leurgans' algorithm for tensor decomposition}
   \label{alg: Leurgans}
\begin{algorithmic}[1]
   \STATE {\bfseries Input:} Tensor $\vX$
   %\REPEAT 
   \STATE Generate contraction vectors $a, b \in \R^{n_{3}}$ (such that non-degeneracy and pairwise genericity holds).
   \STATE Compute mode $3$ contractions $X_a^3$ and $X_b^3$ respectively.  
\STATE Compute eigen-decomposition of $M_1 := X_a^{3}(X_b^3)^{\dagger}$ and $M_2 :=(X_b^3)^{\dagger}X_a$. Let $U$ and $V$ denote the matrices whose columns are the eigenvectors of $M_1$ and $M_2^T$ respectively corresponding to the non-zero eigenvalues, in sorted order. (Let $r$ be the (common) rank of $M_1$ and $M_2$.) The eigenvectors, thus arranged are denoted as $\left\{ u_i \right\}_{i=1, \ldots, r}$ and $\left\{ v_i \right\}_{i=1, \ldots, r}$. \label{step:order}

\STATE Solve for $w_i$ in the (over-determined) linear system $\vX=\sum_{i=1}^{r} u_i \otimes v_i \otimes w_i , i=1, \ldots, m$.
  \STATE {\bfseries Output:} Decomposition $\vX = \sum_{i=1}^r u_i \otimes v_i \otimes w_i$.
\end{algorithmic}
\end{algorithm}
\begin{remark}
Note that in the last step, instead of solving a linear system of equations to obtain the $w_i$, there is an alternative approach whereby one may compute mode $1$ contractions and then obtain the factors $v_i$ and $w_i$. However, there is one minor caveat.
Suppose we denote the factors obtained from the modal contractions $X_a^3$ and $X_b^3$ by $U$ and $V_1$ (we assume that these factors are normalized, i.e. the columns have unit Euclidean norm). Now, we can repeat the procedure with two more random vectors $c, d$ to compute the contractions $X_c^{1}$ and $X_d^{1}$. We can perform similar manipulations to construct matrices whose eigenvectors are the tensor factors of interest, and thence obtain
(normalized) factors $V_2$ and $W$. While $V_1$ and $V_2$ essentially correspond to the same factors, the matrices themselves may (i) have their columns in different order, and (ii) have signs reversed relative to each other. Hence, while the modal contractions preserve information about the tensor factors, they may need to be properly aligned by rearranging the columns and performing sign reversals, if necessary. 
%\textcolor{red}{This is possible by pairing the eigenvalues as we have been doing correct?}
\end{remark}

\subsection{High Level Approach}
%True solution property. \\
%Faithfulness. \\
%Meta-theorem: whenever $\mathcal{T}$ is faithful wrt to $\vX$, we can recover the tensor. \\
%Faithfulness under natural assumptions. \\

The key observation driving the methodology concerns the separability of the measurements. Given a set of separable measurements $y=\mathcal{L} \left( \vX \right)$, from the definition of separability we have:
 \begin{align*}
 y &=\mathcal{L}\left( \vX \right)  = \sum_{i=1}^{n_3} w_i \mathcal{T} \left( X^3_i \right) = \mathcal{T} \left(  \sum_{i=1}^{n_3}  w_iX^3_i \right) = \mathcal{T} \left( X^3_w \right).
 \end{align*}
In words, each separable measurement $\mathcal{L}$ acting on the tensor can also be interpreted as a measurement $\mathcal{T}$ acting on a \emph{contraction} of the tensor. Since these contractions are low rank (Lemma \ref{lemma:contr_rank}), when the underlying tensor is low-rank, the following nuclear norm minimization problem represents a principled, tractable heuristic for recovering the contraction:
\begin{align*}
{\operatorname*{minimize}}_Z\  \|Z\|_* \qquad
\text{subject to}  \qquad y= \mathcal{T} \left( Z\right).
\end{align*}

Let us informally define $\mathcal{T}$ to be ``faithful'' if nuclear norm minimization succeeds in exactly recovering the tensor contractions. Provided we correctly recover two contractions each along modes $1$ and $3$, and furthermore these contractions are non-degenerate and pairwise generic, we can apply Leurgans' algorithm to the recovered contractions to exactly recover the unknown tensor. This yields the following meta-theorem:

\begin{metatheorem} Given a low rank tensor $\vX$ and separable measurements 
$$
y_k^{(i)}=\mathcal{L}^{(i)}_k\left( \vX \right) =  \sum_{j=1}^{n_3} \left(w_k^{(i)}\right)_j \mathcal{T}^{(i)} \left( X^3_j \right), \qquad i\in \left\{1,3 \right\}, \;k \in \left\{1,2 \right\}.
$$
Suppose the $\mathcal{T}^{(i)}$ are faithful and for the vectors $w_k^{(i)}$, the contractions $X^{i}_{w_k^{(i)}}$ are non-degenerate and pairwise generic. Then the proposed approach succeeds in exactly recovering the unknown tensor.
\end{metatheorem}

In the next section, we will make the above meta-theorem more precise, and detail the precise sample complexities for the separable random projections and tensor completion settings. We will see that faithfulness, non-degeneracy and pairwise genericity hold naturally in these settings.

%\begin{proof}
%Suppose we are given an order 3 tensor $\vX = \sum_{i = 1}^r u_i \otimes v_i \otimes w_i$ of size $n_1 \times n_2 \times n_3$.
%From the definition of contraction \eqref{eq:contr_def}, it is straightforward to see that 
%\[
%X_a^3 = UD_aV^T ~\ D_a = \mbox{diag}(a^Tw_1, \ldots, a^Tw_n)
%\]
%\[
%X_b^3 = UD_bV^T ~\ D_b = \mbox{diag}(b^Tw_1, \ldots, b^Tw_n).
%\]
%In the above decompositions, $U$ is $n_1 \times r$, $V$ is $n_2 \times r$, $D_a, D_b$ are $r \times r$.
%Now, 
%\begin{align}
%\notag
%M_1 &:= X_a^3 (X_b^3)^\dagger \\
%\notag
%&= UD_aV^T (V^{\dagger})^{T}D_b^{-1}U^{\dagger} \\
%\label{udiag}
%&= UD_aD_b^{-1}U^\dagger
%\end{align}
%and similarly we obtain
%\begin{equation}
%\label{vdiag}
%M_2 = VD_b^{-1}D_aV^\dagger.
%\end{equation}
%Since we have $M_1U=U D_aD_b^{-1} $ and $M_2V=VD_b^{-1} D_a$, it follows that the columns of $U$ and $V$ are eigenvectors of $M_1$ and $M_2$ respectively (with corresponding eigenvalues given by the diagonal matrices $D_aD_b^{-1}$ and $D_b^{-1}D_a$).
%
%Now, by the randomness of the vectors $a$ and $b$, we it is easy to see that the diagonal entries of the matrix $D_a D_b^{-1}$ are distinct almost surely. It then follows that the eigendecompositions of $M_1, ~ M_2$ are unique, and we can pair the columns of $U$ and $V$ based on their eigenvalues. 
%\end{proof}

\section{Sample Complexity Results: Third Order Case} \label{sec:third_order}
\label{sec:mainresults}
\subsection{Tensor Recovery via Contractions} \label{sec:trecs}
We start by describing the main algorithm of this paper more precisely: Tensor Recovery via Contractions (T-ReCs). 
We assume that we are given separable measurements $y^{(3)}_1=\mathcal{L}^{(3)}_1\left( \vX \right)$, 
   $y^{(3)}_2=\mathcal{L}^{(3)}_2 \left( \vX \right)$, $y^{(1)}_1=\mathcal{L}^{(1)}_1\left( \vX \right)$, $y^{(1)}_2=\mathcal{L}^{(1)}_2 \left( \vX \right)$. We further assume that the measurements are separable as:
  
 \begin{equation} \label{eq:rand_proj}
 \begin{split}
 \mathcal{L}^{(3)}_1\left( \vX \right) = \sum_{i=1}^{n_3} a_i \mathcal{T}^{(3)} \left( X^3_i \right) \qquad  \mathcal{L}^{(3)}_2\left( \vX \right) = \sum_{i=1}^{n_3} b_i \mathcal{T}^{(3)} \left( X^3_i \right) \\
  \mathcal{L}^{(1)}_1\left( \vX \right) = \sum_{i=1}^{n_1} c_i \mathcal{T}^{(1)} \left( X^1_i \right) \qquad  \mathcal{L}^{(1)}_2\left( \vX \right) = \sum_{i=1}^{n_1} d_i \mathcal{T}^{(1)} \left( X^1_i \right). 
  \end{split}
 \end{equation}
 where $a,b,c,d$ and $\mathcal{T}^{(3)}$ and  $\mathcal{T}^{(1)}$ are known in advance.
 Given these measurements our algorithm will involve the solution of the following convex optimization problems.
 \begin{equation} \label{eq:opt1}
\underset{Z_1}{\text{minimize}} \qquad  \|Z_1\|_*  \qquad
\text{s.t.} \qquad  y^{(3)}_1=\mathcal{T}^{(3)} \left( Z_1 \right)
\end{equation}
\begin{equation} \label{eq:opt2}
\underset{Z_2}{\text{minimize}} \qquad  \|Z_2\|_* \qquad
\text{s.t.} \qquad  y^{(3)}_2=\mathcal{T}^{(3)} \left( Z_2 \right)
\end{equation}
\begin{equation} \label{eq:opt3}
\underset{Z_3}{\text{minimize}} \qquad  \|Z_3\|_* \qquad
\text{s.t.} \qquad  y^{(1)}_1=\mathcal{T}^{(1)} \left( Z_3 \right)
\end{equation}
\begin{equation} \label{eq:opt4}
\underset{Z_4}{\text{minimize}} \qquad  \|Z_4\|_* \qquad
\text{s.t.} \qquad  y^{(1)}_2=\mathcal{T}^{(1)} \left( Z_4 \right)
\end{equation}

Efficient computational methods have been extensively studied in recent years for solving problems of this type \cite{nuc_norm}. These matrices form the ``input matrices" in the next step which is an adaptation of Leurgans' method. In this step we form eigendecompositions to reconstruct first the pair of factors $u_i, v_i$, and then the pairs $v_i, w_i$ (the factors are normalized). Once these are recovered the last step involves solving a linear system of equations for the weights $\lambda_i$ in 
\[
\mathcal{L}^{(3)}_1(\vX) = \sum_{i=1}^r \lambda_i \mathcal{L}^{(3)}_1(u_i \otimes v_i \otimes w_i) = y_1^{(3)}
\]
%$\langle A^{(i)} \otimes a ,\vT^* \rangle = \sum_{i=1}^r \lambda_i \langle A^{(i)} \otimes a , u_i \otimes v_i \otimes w_i \rangle \; i=1, \ldots, m_1$. in the case of tensor recovery from random samples, or for $\lambda_i$ in $\vX_{\Omega}=  \sum_{i=1}^{r}  \lambda_i \left( u_i \otimes v_i \otimes w_i \right)_{\Omega} $ in the case of tensor completion. 
The pseudocode for T-ReCs is detailed in Algorithm \ref{alg:trecs}. 

% The key observation driving the algorithm concerns the separability of the measurements. Consider for example the measurements:
% \begin{align*}
% y^{(3)}_1 &=\mathcal{L}^{(3)}_1\left( \vX \right)  \\
% &= \sum_{i=1}^{n_3} a_i \mathcal{T}^{(3)} \left( X^3_i \right) \\
% &= \mathcal{T}^{(3)} \left(  \sum_{i=1}^{n_3}  a_iX^3_i \right) \\
% &= \mathcal{T}^{(3)} \left( X^3_a \right).
% \end{align*}
%In words, each separable measurement $\mathcal{L}$ acting on the tensor can also be interpreted as a measurement $\mathcal{T}$ acting on a \emph{contraction} of the tensor. Since these contractions are low rank when the underlying tensor is low-rank, nuclear norm minimization is likely to recover the contraction itself (provided the operator $\mathcal{T}$ has suitable properties). One only needs to recover a few (four) contractions to then be able to recover the whole tensor from Leurgans' algorithm.

\begin{algorithm}[!ht]
   \caption{Tensor-Recovery via Contractions \hspace{5mm}(T-ReCs)}
   \label{alg:trecs}
\begin{algorithmic}[1]
   \STATE {\bfseries Input:} Separable measurements $y^{(3)}_1=\mathcal{L}^{(3)}_1\left( \vX \right)$, 
   $y^{(3)}_2=\mathcal{L}^{(3)}_2 \left( \vX \right)$, $y^{(1)}_1=\mathcal{L}^{(1)}_1\left( \vX \right)$, $y^{(1)}_2=\mathcal{L}^{(1)}_2 \left( \vX \right)$. 
   %\REPEAT 
   \STATE Solve convex optimization problems \eqref{eq:opt1} and \eqref{eq:opt2} to obtain optimal solutions $Z_1^*$ and $Z_2^*$ respectively.  
\STATE Compute eigen-decomposition of $M_1 := Z_1^{*}(Z_2^*)^{\dagger}$ and $M_2 := (Z_2^*)^{\dagger}Z_1$. Let $U$ and $V$ denote the matrices whose columns are the eigenvectors of $M_1$ and $M_2^T$ respectively corresponding to the non-zero eigenvalues, in sorted order. (Let $r$ be the (common) rank of $M_1$ and $M_2$.) The eigenvectors, thus arranged are denoted as $\left\{ u_i \right\}_{i=1, \ldots, r}$ and $\left\{ v_i \right\}_{i=1, \ldots, r}$.

\STATE Solve convex optimization problems \eqref{eq:opt3} and \eqref{eq:opt4} to obtain optimal solutions $Z_3^*$ and $Z_4^*$ respectively.  
\STATE Compute eigen-decomposition of $M_3 := Z_3^{*}(Z_4^*)^{\dagger}$ and $M_4 := (Z_4^*)^{\dagger}Z_3$. Let $\tilde{V}$ and $\tilde{W}$ denote the matrices whose columns are the eigenvectors of $M_3$ and $M_4^T$ respectively corresponding to the non-zero eigenvalues, in sorted order. (Let $r$ be the (common) rank of $M_3$ and $M_4$.) The eigenvectors, thus arranged are denoted as $\left\{ \tilde{v}_k \right\}_{k=1, \ldots, r}$ and $\left\{ \tilde{w}_k \right\}_{k=1, \ldots, r}$.

\STATE Simultaneously reorder the columns of $\tilde{V}, \tilde{W}$, also performing simultaneous sign reversals as necessary so that the columns of $V$ and $\tilde{V}$ are equal, call the resulting matrix $W$ with columns $\left\{  w_{i} \right\}_{i=1, \ldots, r}$.

\STATE Solve for $\lambda_i$ in the (over-determined) linear system $$y_i=  \sum_{i=1}^{r} \lambda_i \mathcal{L} \left( u_i \otimes v_i \otimes w_i \right).$$
  \STATE {\bfseries Output:} Recovered tensor $\vX=\sum_{i=1}^r  \lambda_i \, u_i \otimes v_i \otimes w_i$.
\end{algorithmic}
\end{algorithm}
We now focus on the case of recovery from random Gaussian measurements, and then move on to the case of recovery from partially observed samples - in these situations not only are the measurements separable but one can also obtain provable sample complexity bounds which are almost optimal.

%%%%%%%%%%%%%%%%%%%%%%%%%%%%%%%%%%%%%%%%%%%%%%%%%%%%%
%%%%%%%%%%%%%%%%%%%%%%%%%%%%%%%%%%%%%%%%%%%%%%%%%%%%%
%%%%%%%%%%%%%%%%%%%%%%%%%%%%%%%%%%%%%%%%%%%%%%%%%%%%%
%%%%%%%%%%%%%%%%%%%%%%%%%%%%%%%%%%%%%%%%%%%%%%%%%%%%%
%%%%%%%%%%%%%%%%%%%%%%%%%%%%%%%%%%%%%%%%%%%%%%%%%%%%%
%%%%%%%%%%%%%%%%%%%%%%%%%%%%%%%%%%%%%%%%%%%%%%%%%%%%%
%%%%%%%%%%%%%%%%%%%%%%%%%%%%%%%%%%%%%%%%%%%%%%%%%%%%%
%%%%%%%%%%%%%%%%%%%%%%%%%%%%%%%%%%%%%%%%%%%%%%%%%%%%%
%%%%%%%%%%%%%%%%%%%%%%%%%%%%%%%%%%%%%%%%%%%%%%%%%%%%%
%%%%%%%%%%%%%%%%%%%%%%%%%%%%%%%%%%%%%%%%%%%%%%%%%%%%%
%%%%%%%%%%%%%%%%%%%%%%%%%%%%%%%%%%%%%%%%%%%%%%%%%%%%%
%%%%%%%%%%%%%%%%%%%%%%%%%%%%%%%%%%%%%%%%%%%%%%%%%%%%%
%%%%%%%%%%%%%%%%%%%%%%%%%%%%%%%%%%%%%%%%%%%%%%%%%%%%%
%%%%%%%%%%%%%%%%%%%%%%%%%%%%%%%%%%%%%%%%%%%%%%%%%%%%%
%%%%%%%%%%%%%%%%%%%%%%%%%%%%%%%%%%%%%%%%%%%%%%%%%%%%%

\subsection{Separable Random Projections}\label{sec:random_sensing}

Recall that from the discussion in Section \ref{sec:prelim} and the notation introduced in Section \ref{sec:diversity}, we have the following set of measurements:
\begin{align*}
\mathcal{L}_{1}^{(3)}\left( \vX \right) = \left[ \begin{array}{c}
\langle A_1 \otimes a, \vX \rangle \\
\vdots \\
\langle A_{m_1} \otimes a, \vX \rangle
\end{array} \right], \qquad
\mathcal{L}_{2}^{(3)}\left( \vX \right) = \left[ \begin{array}{c}
\langle A_{1} \otimes b, \vX \rangle \\
\vdots \\
\langle A_{m_1} \otimes b, \vX \rangle
\end{array} \right],
\end{align*}

\begin{align*}
\mathcal{L}_{1}^{(1)}\left( \vX \right) = \left[ \begin{array}{c}
\langle c \otimes B_{1}, \vX \rangle \\
\vdots \\
\langle c \otimes B_{m_2}, \vX \rangle
\end{array} \right], \qquad
\mathcal{L}_{2}^{(1)}\left( \vX \right) = \left[ \begin{array}{c}
\langle d \otimes B_{1}, \vX \rangle \\
\vdots \\
\langle d \otimes B_{m_2}, \vX \rangle
\end{array} \right].
\end{align*}

In the above, each $A_{i}, B_i \in \R^{n_2 \times n_3}$ is a random Gaussian matrix with i.i.d $\mathcal{N}(0,1)$ entries, and $a,b \in \R^{n_3}, ~\ c,d \in \R^{n_1}$ are random vectors distributed uniformly on the unit sphere. 
Finally, collecting all of the above measurements into a single operator, we have $y=\mathcal{L} \left( \vX \right)$, and the total number of samples is thus $m = 2m_1+2m_2$.

In the context of random tensor sensing, \eqref{eq:opt1}, \eqref{eq:opt2}, \eqref{eq:opt3} and \eqref{eq:opt4} reduce to solving low rank matrix recovery problems from random Gaussian measurements, where the measurements are as detailed in Section \ref{sec:separable}.

The following lemma shows that the observations $\mathcal{L}\left( \vX \right)$ can essentially be thought of as linear Gaussian measurements of the contractions $X_a^3, X_b^3, X_c^1, X_d^1$. This is crucial in reducing the tensor recovery problem to the problem of recovering the tensor contractions, instead.

\begin{lemma} \label{lemma:main_obs}
For tensor $\vX$, matrix $A$ and vector $a$ of commensurate dimensions,
\begin{equation*}
\langle A \otimes a, \vX \rangle = \langle A, X^{3}_a\rangle.
\end{equation*}
Similarly, for a vector $c$ and matrix $B$ of commensurate dimensions
\begin{equation*}
\langle c \otimes B, \vX \rangle = \langle B, X^{1}_c\rangle.
\end{equation*}
\end{lemma}
\begin{proof}
We only verify the first equality, the second equality is proved in an identical manner.
Let us denote by $\vX_k$ the $k^{th}$ mode $3$ slice of $\vX$ where $k=1, \ldots, n_3$. Then we have,
\begin{align*}
\langle A \otimes a, \vX \rangle = & \sum_{k=1}^{n_3} a_k \langle A, \vX_k \rangle = \langle A, \sum_{k=1}^{n_3} a_k \vX_k \rangle = \langle A, X^{3}_a\rangle.
\end{align*}
\end{proof}

As a consequence of the above lemma, it is easy to see that 
$$
\langle A\otimes a, \vX \rangle= \langle A, X^{3}_a\rangle =  \langle A, \sum_{i=1}^{n_3} a_iX^{3}_i\rangle = \sum_{i=1}^{n_3} a_i \langle A, X^{3}_i\rangle,
$$
thus establishing separability.

Since $X^{3}_a$ and $X^{3}_b$ are low-rank matrices, the observation operators $\mathcal{L}_k^{(3)}\left( \vX \right)$ essentially provide Gaussian random projections of $X^{3}_a$ and $X^{3}_b$, which in turn can be recovered using matrix-based techniques. The following lemma establishes ``faithfulness'' in the context of separable random projections.

\begin{lemma} \label{lemma:contr_rec}
Suppose $m_1 > 3r(n_1+n_2-r)$. Then the unique solutions to problems \eqref{eq:opt1} and \eqref{eq:opt2} are $X^{3}_a$ and $X^{3}_b$ respectively with high probability. Similarly, if $m_2 > 3r(n_2+n_3-r)$ then the unique solutions to problems \eqref{eq:opt3} and \eqref{eq:opt4} are $X^{1}_c$ and $X^{1}_d$ respectively with high probability.
\end{lemma}
%\textcolor{red}{GT: needs to be rephrased if the comment below equation (2) is addressed. \\
%Pari: Making these changes seems to make the notation more complex. So instead, I just made a comment earlier saying that the $\mathcal{T}^{(i)}$ need not be the same. Do you nevertheless want go ahead with this change?}

\begin{proof}
Again, we only prove the first part of the claim, the second follows in an identical manner.
Note that by Lemma \ref{lemma:main_obs} and Lemma \ref{lemma:contr_rank}, $X^{3}_a$ and $X^{3}_b$ are feasible rank $r$ solutions to \eqref{eq:opt1} and \eqref{eq:opt2} respectively. By Proposition 3.11 of \cite{venkat}, we have that the nuclear norm heuristic succeeds in recovering rank $r$ matrices from $m_1 >3r(n_1+n_2-r)$ with high probability.
\end{proof}
\begin{remark}
In this sub-section, we will refer to events which occur with probability exceeding $1-\exp(-C_0 n_1)$ as events that occur ``with high probability" (w.h.p.). We will transparently be able to take appropriate union bounds of high probability events since the number of events being considered is small enough that the union event also holds w.h.p. (thus affecting only the constants involved). Hence, in the subsequent results, we will not need to refer to the precise probabilities.
\end{remark}

Since the contractions $X^{3}_a$ and $X^{3}_b$ of the tensor $\vX$ are successfully recovered and the tensor satisfies Assumption \ref{assump:1}, the second stage of Leurgans' algorithm can be used to recover the factors $u_i$ and $v_i$. Similarly, from $X^{1}_c$ and $X^{1}_d$, the factors $v_i$ and $w_i$ can be recovered.  The above sequence of observations leads to the following sample complexity bound for low rank tensor recovery from random measurements:

\begin{theorem} \label{thm:exact_rec_sense}
Let $\vX \in \R^{n_1\times n_2 \times n_3}$ be an unknown tensor of interest with rank $r \leq \min\left\{n_1,n_2, n_3\right\}$. Suppose we obtain samples as described by \eqref{eq:rand_proj}. Suppose $m_1 > 3r(n_1+n_2-r)$ and $m_2 > 3r(n_2+n_3-r)$. Then T-ReCs (Algorithm \ref{alg:trecs}) succeeds in exactly recovering $\vX$ and its low rank decomposition \eqref{eq:decomp} with high probability.
\end{theorem}

%\textcolor{red}{GT: needs to be rephrased if the comment below equation (2) is addressed.}

\begin{proof}
By Lemma \ref{lemma:contr_rank} $X^{3}_a$, $X^{3}_b$, $X^{1}_c$, $X^1_d$ are all rank at most $r$. By Lemma \ref{lemma:main_obs}, the tensor observations $y^{(3)}_1, y^{(3)}_2, y^{(1)}_1, y^{(1)}_2$ provide linear Gaussian measurements of $X^{3}_a$, $X^{3}_b$, $X^{1}_c$, $X^1_d$. By Lemma \ref{lemma:contr_rec}, the convex problems \eqref{eq:opt1}, \eqref{eq:opt2}, \eqref{eq:opt3}, \eqref{eq:opt4} correctly recover the modal contractions $X^{3}_a$, $X^{3}_b$, $X^{1}_c$, $X^1_d$. Since the vectors $a, b, c, d$ are chosen to be randomly uniformly distributed on the unit sphere, the contractions $X^3_a, X^3_b$ are non-degenerate and pairwise generic almost surely (and similarly $X^1_c, X^1_d$). Thus, Lemma \ref{lemma:fund_lemma} applies and $X^{3}_a$, $X^{3}_b$ can be used to correctly recover the factors $u_i, v_i, i=1, \ldots, r$. Again by Lemma \ref{lemma:fund_lemma} $X^{1}_c$, $X^{1}_d$ can be used to correctly recover the factors $v_i, w_i, i=1, \ldots, r$. Note that due to the linear independence of the factors, the  linear system of equations involving $\lambda_i$ is full column rank, over-determined, and has an exact solution. The fact that the result holds with high probability follows because one simply needs to take the union bounds of the probabilities of failure exact recovery of the contractions via the solution of  \eqref{eq:opt1}, \eqref{eq:opt2}, \eqref{eq:opt3}, \eqref{eq:opt4}.
\end{proof}
%\vspace{-3mm}
\begin{remarks} \hspace{10mm}
\begin{enumerate}
\item Theorem \ref{thm:exact_rec_sense} yields bounds that are order optimal. Indeed, consider the number of samples $m = 2m_1 + 2m_2 \sim O(r(n_1 + n_2 + n_3))$, which by a counting argument is the same as the number of parameters in an order 3 tensor of rank $r$. 

\item For symmetric tensors with symmetric factorizations of the form $\vX=\sum_{l=1}^{3} \lambda_i v_i \otimes v_i \otimes v_i$, this method becomes particularly simple. Steps $4, 5, 6$ in Algorithm \ref{alg:trecs} become unnecessary, and the factors are revealed directly in step $3$. One then only needs to solve the linear system described in step $7$ to recover the scale factors. The sample complexity remains $O(nr)$, nevertheless.

\item Note that for the method we propose, the most computationally expensive step is that of solving low-rank matrix recovery problems where the matrix is of size $n_i \times n_j$ for $i,j = 1,2,3$. Fast algorithms with rigorous guarantees exist for solving such problems, and we can use any of these pre-existing methods. An important point to note is that, other methods for minimizing the Tucker rank of a tensor by considering ``matricized" tensors solve matrix recovery problems for matrices of size $n_i \times n_j n_k$, which can be far more expensive. 

\item Note that the sensing operators $\langle A_i\otimes a , \cdot \rangle$ may seem non-standard (vis-a-vis the compressed sensing literature such as \cite{recht2010guaranteed}), but are very storage efficient. Indeed, one needs to only store random matrices $A_i, B_i$ and random vectors $a,b$. Storing each of these operators requires $O(n_1 n_2 + n_3)$ space, and is far more storage efficient than (perhaps the more suggestive) sensing operators of the form $\langle \vA_i , \cdot \rangle$, with each $\vA_i$ being a random tensor requiring $O(n_1 n_2 n_3)$ space. Similar ``low rank" sensing operators have been used for matrix recovery \cite{kueng2014low, jain2013provable}.

\item While the results here are presented in the case where the $A_i, B_i$ are random Gaussian matrices and the $a,b$ are uniformly distributed on the sphere, the results are not truly dependent on these distributions. The $A_i, B_i$ need to be structured so that they enable low-rank matrix recovery (i.e., they need to be ``faithful''). Hence, for instance it would suffice if the entries of these matrices were sub-Gaussian, or had appropriate restricted isometry properties with respect to low rank matrices \cite{recht2010guaranteed}.  
\end{enumerate}
\end{remarks}

%%%%%%%%%%%%%%%%%%%%%%%%%%%%%%%%%%%%%%%%%%%%%%%%
%%%%%%%%%%%%%%%%%%%%%%%%%%%%%%%%%%%%%%%%%%%%%%%%
%%%%%%%%%%%%%%%%%%%%%%%%%%%%%%%%%%%%%%%%%%%%%%%%
%%%%%%%%%%%%%%%%%%%%%%%%%%%%%%%%%%%%%%%%%%%%%%%%
%%%%%%%%%%%%%%%%%%%%%%%%%%%%%%%%%%%%%%%%%%%%%%%%
%%%%%%%%%%%%%%%%%%%%%%%%%%%%%%%%%%%%%%%%%%%%%%%%
%%%%%%%%%%%%%%%%%%%%%%%%%%%%%%%%%%%%%%%%%%%%%%%%
%%%%%%%%%%%%%%%%%%%%%%%%%%%%%%%%%%%%%%%%%%%%%%%%
%%%%%%%%%%%%%%%%%%%%%%%%%%%%%%%%%%%%%%%%%%%%%%%%
%%%%%%%%%%%%%%%%%%%%%%%%%%%%%%%%%%%%%%%%%%%%%%%%
\subsection{Tensor Completion} \label{sec:third_order_completion}

In the context of tensor completion, for a fixed (but unknown) $\vX $, a subset of the entries $\vX_\Omega$ are revealed for some index set $\Omega \subseteq [n_1] \times [n_2] \times [n_3]$. We assumed that the measurements thus revealed are in a union of four slices. For the $i^{th}$ mode-$1$ slice let us define 
\begin{align*}
\Omega^{(1)}_i:=\Omega \cap S^{(1)}_i \qquad m^{(1)}_i:=\left| \Omega \cap S_i^{(1)} \right|.
\end{align*}
These are precisely the set of entries revealed in the $i^{th}$ mode-$1$ slice and the corresponding cardinality.
Similarly for the $k^{th}$ mode-$3$ slice we define
\begin{equation} \label{eq:comp_meas1}
\Omega^{(3)}_k=\Omega \cap S^{(3)}_k \qquad m^{(3)}_k:=\left| \Omega \cap  S_k^{(3)} \right|.
\end{equation}
We will require the existence of two distinct mode-$1$ slices (say $i^*_1$ and $i^*_2$) from which measurements are obtained. Indeed, 
%$$\mathcal{L}^{(1)}\left( \vX \right) = 
%\left[ \begin{array}{c}
%\mathcal{L}^{(1)}_1 \left( \vX \right) \\
%\mathcal{L}^{(1)}_2 \left( \vX \right)
%\end{array} \right],
% $$
% where
 \begin{equation} \label{eq:comp_meas2}
\mathcal{L}^{(1)}_1 \left( \vX \right):=\left( \vX \right)_{\Omega^{(1)}_{i_1^*}} \qquad  \mathcal{L}^{(1)}_2 \left( \vX \right):=\left( \vX \right)_{\Omega^{(1)}_{i_2^*}}.
 \end{equation}
Similarly we will also require the existence of two different slices in mode $3$ \footnote{We choose modes 1 and 3 arbitrarily. Any two of the 3 modes suffice. } (say $k^*_1$ and $k^*_2$) from which we have measurements:
%$$\mathcal{L}^{(1)}\left( \vX \right) = 
%\left[ \begin{array}{c}
%\mathcal{L}^{(3)}_1 \left( \vX \right) \\
%\mathcal{L}^{(3)}_2 \left( \vX \right)
%\end{array} \right],
% $$
% where
 $$
\mathcal{L}^{(3)}_1 \left( \vX \right):=\left( \vX \right)_{\Omega^{(3)}_{k_1^*}} \qquad  \mathcal{L}^{(3)}_2 \left( \vX \right):=\left( \vX \right)_{\Omega^{(3)}_{k_2^*}}.
 $$

We will require the cardinalities of the measurements from mode $1$, $m^{(1)}_{i_{1}^{*}}$ and $m^{(1)}_{i_{2}^{*}}$ and from mode $3$, $m^{(3)}_{k_{1}^{*}}$ and $m^{(3)}_{k_{2}^{*}}$ to be sufficiently large so that they are faithful (to be made precise subsequently), and this will determine the sample complexity. The key aspect of the algorithm is that it \emph{only makes use of the samples in these four distinct slices}. No other samples outside these four slices need be revealed at all (so that all the other $m^{(1)}_i$ and $m^{(3)}_k$ can be zero). The indices sampled from each slice are drawn uniformly and randomly without replacement. Note that for a specified $m^{(1)}_{i_{1}^{*}}$, $m^{(1)}_{i_{2}^{*}}$, $m^{(3)}_{k_{1}^{*}}$ and $m^{(3)}_{k_{2}^{*}}$ the overall sample complexity implied is  $m^{(1)}_{i_{1}^{*}} + m^{(1)}_{i_{2}^{*}} + m^{(3)}_{k_{1}^{*}} + m^{(3)}_{k_{2}^{*}}$.

In the context of tensor completion, \eqref{eq:opt1}, \eqref{eq:opt2}, \eqref{eq:opt3} and \eqref{eq:opt4} reduce to  solving low rank matrix completion problems for  the slices $S_{i_{1}^{*}}^{(1)}$, $S_{i_{2}^{*}}^{(1)}$, $S_{k_{1}^{*}}^{(3)}, S_{k_{2}^{*}}^{(3)}$.  Contraction recovery in this context amounts to obtaining complete slices, which can then be used as inputs to Leurgans' algorithm. 
%A variety of matrix completion techniques can be used for this purpose, but we will use the one advocated in \cite{Ben}, which employs the nuclear norm minimization heuristic. A key innovation in this paper is to show that from only this limited information the entire underlying tensor can be reconstructed. 
There are a few important differences however, when compared to the case of recovery from Gaussian random projections. For the matrix completion sub-steps to succeed, we need the following standard incoherence assumptions from the matrix completion literature \cite{Ben}.

Let $\mathcal{U}, \mathcal{V}$ and $\mathcal{W}$ represent the linear spans of the vectors $\left\{ u_i \right\}_{=1, \ldots, r}, \left\{ v_i \right\}_{=1, \ldots, r}, \left\{ w_i \right\}_{=1, \ldots, r}$. Let $P_{\mathcal{U}}$, $P_{\mathcal{V}}$ and $P_{\mathcal{W}}$ respectively represent the projection operators corresponding to $\mathcal{U}, \mathcal{V}$ and $\mathcal{W}$. The coherence of the subspace $\mathcal{U}$ (similarly for $\mathcal{V}$ and $\mathcal{W}$) is defined as:
$$
\mu(\mathcal{U}):=\frac{n_1}{r} \max_{i=1, \ldots, n_1} \| P_{\mathcal{U}}\left( e_i \right) \|^2,
$$
where $\{e_i\}$ are the canonical basis vectors.
\begin{assumption}[Incoherence] \label{assump:3}
 $\mu_0:= \max\left\{ \mu(\mathcal{U}), \mu(\mathcal{V}), \mu(\mathcal{W})  \right\}$ is a positive constant independent of the rank and the dimensions of the tensor.
\end{assumption}
Such an incoherence condition is required in order to be able to complete the matrix slices from the observed data \cite{Ben}. We will see subsequently that when the tensor is of rank $r$, so are the different slices of the tensor and each slice will have a ``thin'' singular value decomposition. Furthermore, the incoherence assumption will also hold for these slices. 
\begin{definition} \label{def:1}
 Let $X_i^{1} = U \Sigma V^T$ be the singular value decomposition of the tensor slice $X_i^{1}$. We say that the tensor $\vX$ satisfies the \emph{slice condition} for slice $S^{(1)}_i$ with constant $\mu^{(1)}_i$ if the element-wise infinity (max) norm $$\|UV^T\|_{\infty} \leq \mu^{(1)}_i \sqrt{\frac{r}{n_2n_3}}.$$
\end{definition} 
The slice condition is analogously defined for the slices along other modes, i.e. $S^{(2)}_j$ and $S^{(3)}_k$. We will denote by $\mu^{(2)}_j$ and $\mu^{(3)}_k$ the corresponding slice constants. We will require our distinct slices  from which samples are obtained to satisfy these slice conditions. 
\begin{remark}
The slice conditions are standard in the matrix completion literature, see for instance \cite{Ben}.
 As pointed out in \cite{Ben}, the slice conditions are not much more restrictive than the incoherence condition, because if the incoherence condition is satisfied with constant $\mu_0$ then (by a simple application of the Cauchy-Schwartz inequality) the slice condition for $S_{i}^{(1)}$ is also satisfied with constant $\mu_1(i) \leq \mu_0 \sqrt{r}$ for all $i$ (and similarly for $\mu^{(2)}_j$ and $\mu^{(3)}_k$). Hence, the slice conditions can be done away with, and using this weaker bound only increases the sample complexity bound for exact reconstruction by a multiplicative factor of $r$.
\end{remark}

\begin{remark}
Note that the incoherence assumption and the slice condition are known to be satisfied for suitable random ensembles of models, such as the random orthogonal model, and models where the singular vectors are bounded element-wise \cite{Ben}.
\end{remark}

The decomposition \eqref{eq:decomp} ties factor information about the tensor to factor information of contractions. A direct corollary of Lemma \ref{lemma:contr_rank} is that contraction matrices are incoherent whenever the tensor is incoherent:
\begin{corollary} \label{cor:incoherence}
If the tensor satisfies the incoherence assumption, then so do the contractions. Specifically all the tensor slices satisfy incoherence.
\end{corollary}
\begin{proof}
Consider for instance the slices  $X_k^3$ for $k=1, \ldots, n_3$.
By Lemma \ref{lemma:contr_rank}, the row and column-spaces of each slice are precisely $\mathcal{U}$ and $\mathcal{V}$ respectively, thus the incoherence assumption also holds for the slices.
\end{proof}

We now detail our result for the tensor completion problem:

\begin{lemma} \label{lemma:contr_rec}
Given a tensor $\vX$ with rank $r \leq n_1$ which satisfies the following:
\begin{itemize} 
\item Assumptions \ref{assump:2} and \ref{assump:3}, 
\item  The samples are obtained as described in \eqref{eq:comp_meas1}, \eqref{eq:comp_meas2}.
\item Suppose the number of samples from each slice satisfy:
\begin{align*}
\Ma \geq 32 \max \left\{ \mu_0, \left( \mu_{i_1^{*}}^{(1)} \right)^2 \right\} r(n_2+n_3) \log^{2}n_3 \\
\Mb \geq  32 \max \left\{ \mu_0, \left( \mu_{i_2^{*}}^{(1)} \right)^2 \right\} r(n_2+n_3) \log^{2}n_3 \\
\Mc \geq 32 \max \left\{ \mu_0, \left( \mu_{k_1^{*}}^{(3)} \right)^2 \right\} r(n_1+n_2) \log^{2}n_2 \\
\Md \geq  32 \max \left\{ \mu_0, \left( \mu_{k_2^{*}}^{(3)} \right)^2 \right\} r(n_1+n_2) \log^{2}n_2 \\
\end{align*} 
\vspace{-13mm}
\item The slice condition (Definition \ref{def:1}) for each of the four slices $\Sa, \Sb, \Sc, \Sd$ hold. 
%and non-degeneracy holds for mode $3$ at $k_1^*, k_2^*$ and mode $1$ at $i_1^*, i_2^*$ (Definition \ref{def:degen}).
\end{itemize}
Then the unique solutions to problems \eqref{eq:opt1}, \eqref{eq:opt2}, \eqref{eq:opt3} and \eqref{eq:opt4} are $\Xc$ $\Xd$, $\Xa$, and $\Xb$ respectively with probability exceeding $1-C \log(n_2)n_2^{-\beta}$ for some constants $C, \beta >0$. 
\end{lemma}

\begin{proof}
%Again, we only prove the first part of the claim, the second follows in an identical manner.
By Lemma \ref{lemma:contr_rank} $\Xa$, $\Xb$, $\Xc$, $\Xd$ are all rank at most $r$. By Theorem 1.1 of \cite{Ben}, the convex problems \eqref{eq:opt1}, \eqref{eq:opt2}, \eqref{eq:opt3}, \eqref{eq:opt4} correctly recover the full slices $\Xc$, $\Xd$, $\Xa$, $\Xb$ with high probability. (Note that the relevant incoherence conditions in \cite{Ben} are satisfied due to Corollary \ref{lemma:contr_rank} and the slice condition assumption. Furthermore the number of samples specified meets the sample complexity requirements of Theorem 1.1 in \cite{Ben} for exact recovery.) 
\end{proof}

\begin{remark}
We note that in this sub-section, events that occur with probability exceeding  $1-C \log(n_2) n_2^{-\beta}$ (recall that $n_1 \leq n_2 \leq n_3$) are termed as occurring with high probability (w.h.p.). We will transparently be able to union bound these events (thus changing only the constants) and hence we refrain from mentioning these probabilities explicitly. 
\end{remark}

%Since the contractions $X^{3}_a$ and $X^{3}_b$ of the tensor $\vX$ are successfully recovered, and the tensor satisfies Assumption \ref{assump:1} the second stage of Leurgans' algorithm can be used to recover the factors $u_i$ and $v_i$. Similarly, from $X^{1}_c$ and $X^{1}_d$, the factors $v_i$ and $w_i$ can be recovered.  The above sequence of observations leads to the following:
%
\begin{theorem} \label{thm:exact_rec_comp}
Let $\vX \in \R^{n_1\times n_2 \times n_3}$ be an unknown tensor of interest with rank $r \leq n_1$, such that the tensor slices  $\Xc$ $\Xd$ are non-degenerate and pairwise generic, and similarly $\Xa, \Xb$ are non-degenerate and pairwise generic. Then, under the same set of assumptions made for Lemma \ref{lemma:contr_rec},  the procedure outlined in Algorithm \ref{alg:trecs} succeeds in exactly recovering $\vX$ and its low rank decomposition \eqref{eq:decomp0} with high probability.
\end{theorem}

\begin{proof}
The proof follows along the same lines as that of Theorem \ref{thm:exact_rec_sense}, with Lemma \ref{lemma:contr_rec} allowing us to exactly recover the slices $\Xa, \Xb, \Xc, \Xd$. Since these slices satisfy non-degeneracy and pairwise genericity, the tensor factors $u_i, v_i, w_i$, $i=1, \ldots, r$  can be exactly recovered (up to scaling) by following steps $(3)$, $(5)$ and $(6)$ of Algorithm \ref{alg:trecs}. Also, the system of equations to recover $\lambda$ is given by
\[
\mtx{X}_{\Omega} = \sum_{i = 1}^r \lambda_i (u_i \otimes v_i \otimes w_i)_{\Omega}.
\]
\end{proof}

\begin{remarks} \hspace{10mm}
\begin{enumerate}
\item Theorem \ref{thm:exact_rec_comp} yields bounds that are almost order optimal when $\mu_0$ and $\mu^{(i)}_k$ are constant (independent of $r$ and the dimension). Indeed, the total number of samples required  is $m \sim O(rn_3 \log^2 n_3)$, which by a counting argument is nearly the same number  of parameters in an order 3 tensor of rank $r$ (except for the additional logarithmic factor). 

\item The comments about efficiency for symmetric factorizations in the Gaussian random projections case hold here as well. 

%\item For symmetric tensors with symmetric factorizations of the form $\vX=\sum_{l=1}^{3} \lambda_i v_i \otimes v_i \otimes v_i \in \R^{n \times n \times n}$, this method becomes particularly simple. Steps $4, 5, 6$ in Algorithm \ref{alg:trecs} become unnecessary, and the factors are revealed directly in step $3$. One then only needs to solve the linear system described in step $7$ to recover the scale factors. The sample complexity remains $O(nr \log^2 n)$, nevertheless.

%\item Note that for the method we propose, the most computationally expensive step is solving low-rank matrix recovery problems where the matrix is of size $n_i \times n_j$ for $i,j = 1,2,3$. Fast algorithms with rigorous guarantees exist for solving such problems, and we can use any of these pre-existing methods. An important point to note is that, other methods for minimizing the Tucker rank of a tensor by considering ``matricized" tensors solve matrix recovery problems for matrices of size $n_i \times n_j n_k$, which can be far more expensive. 

%\item Note that the sensing operators $\langle A^{(i)}\otimes a , \cdot \rangle$ may seems non-standard (vis-a-vis the compressed sensing literature such as \cite{recht2010guaranteed}), but are very storage efficient. Indeed, one needs to only store random matrices $A^{(i)}, B^{(i)}$ and random vectors $a,b$. This is far more storage efficient than (perhaps the more suggestive) sensing operators of the form $\langle \vA^{(i)} , \cdot \rangle$, where each $\vA^{(i)}$ is a random tensor. 

\item We do not necessarily need sampling without replacement from the four slices. Similar results can be obtained for other sampling models such as with replacement \cite{Ben}, and even non-uniform sampling \cite{Sujay}. Furthermore, while the method proposed here for the task of matrix completion relies on nuclear norm minimization, a number of other approaches such as alternating minimization \cite{optspace,Sujay,Burer} can also be adopted; our algorithm relies only on the successful completion of the slices. 

\item Note that we can remove the slice condition altogether since the incoherence assumption implies the slice condition with $\mu_1 = \mu_0 \sqrt r$. Removing the slice condition then implies an overall sample complexity of $O(r^2n_3\log^2n_3)$.
\end{enumerate}
\end{remarks}

%%%%%%%%%%%%%%%%%%%%%%%%%%%%%%%%%%%%%%%%%%%%%%%%
%%%%%%%%%%%%%%%%%%%%%%%%%%%%%%%%%%%%%%%%%%%%%%%%
%%%%%%%%%%%%%%%%%%%%%%%%%%%%%%%%%%%%%%%%%%%%%%%%
%%%%%%%%%%%%%%%%%%%%%%%%%%%%%%%%%%%%%%%%%%%%%%%%
%%%%%%%%%%%%%%%%%%%%%%%%%%%%%%%%%%%%%%%%%%%%%%%%
%%%%%%%%%%%%%%%%%%%%%%%%%%%%%%%%%%%%%%%%%%%%%%%%
%%%%%%%%%%%%%%%%%%%%%%%%%%%%%%%%%%%%%%%%%%%%%%%%
%%%%%%%%%%%%%%%%%%%%%%%%%%%%%%%%%%%%%%%%%%%%%%%%
%%%%%%%%%%%%%%%%%%%%%%%%%%%%%%%%%%%%%%%%%%%%%%%%
%%%%%%%%%%%%%%%%%%%%%%%%%%%%%%%%%%%%%%%%%%%%%%%%

\section{Extension to Higher Order Tensors} \label{sec:higher_order}

%\textcolor{red}{I do not think it is necessary to include the algorithms again. We just need to say "here is how the measurements will be obtained" and then port that into TReC-S and FSS. This is looking too repetetive. 
%\begin{itemize}
%\item Define slices
%\item Define contractions
%\item Define separability
%\item Define algorithm
%\item Show that measurement mechanisms are separable, etc.
%\end{itemize}}

The results of Section \ref{sec:third_order} can be extended to higher order tensors in a straightforward way. While the ideas remain essentially the same, the notation is necessarily more cumbersome in this section. We omit some technical proofs to avoid repetition of closely analogous arguments from the third order case, and focus on illustrating how to extend the methods to the higher order setting. 

Consider a tensor $\vX \in \R^{n_{1} \times \cdots \times n_{K}}$ of order $K$ and dimension $n_1\times \cdots \times n_{K}$. Let us assume, without loss of generality, that $n_1 \leq n_2 \leq \ldots \leq n_K$. Let the rank of this tensor be $r \leq n_1$ and be given by the decomposition:
\begin{align*}
\vX &= \sum_{l=1}^r u_l^1 \otimes \cdots \otimes u_l^K = \sum_{l=1}^r \bigotimes_{p=1}^{K} u^p_l,
\end{align*}
where $u_l^p \in \R^{n_p}$.
We will be interested in slices of the given tensor that are identified by picking two \emph{consecutive} modes $(k, k+1)$, and by fixing all the indices not in those modes, i.e. $i_1 \in [n_{1}], \ldots, i_{k-1} \in [n_{k-1}], i_{k+2} \in [n_{k+2}], \ldots, i_K \in [n_K]$. Thus the indices of a slice $S$ are:
\begin{align*}
S:=\left\{ i_1 \right\} \times \cdots  \times \left\{ i_{k-1} \right\} \times [n_k] \times [n_{k+1}] \times \left\{ i_{k+2} \right\} \times \cdots \times \left\{ i_K \right\},
%S^{(k)}:=\left\{ (i_1, \ldots, i_{k-1}, s,t, i_{k+2}, \ldots, i_K) : \right. \\
%\left. s \in [n_k], t \in [n_{k+1}]  \right\}.
 \end{align*}
 and the corresponding slice may be viewed as a matrix, denoted by $\vX_S$. While slices of tensors can be defined more generally (i.e. the modes need not be consecutive), in this paper we will only need to deal with such ``contiguous'' slices. \footnote{In general, a slice corresponding to any pair of modes $(k_1,k_2)$ suffices for our approach. However, to keep the notation simple we present the case where slices correspond to mode pairs of the form $(k,k+1)$.} We will denote the collection of all slices where modes $(k,k+1)$ are contained to be: \small
 $$
 \mathcal{S}^{(k)} := \left\{ 
 \left\{ i_1 \right\} \times \cdots  \times \left\{ i_{k-1} \right\} \times [n_k] \times [n_{k+1}] \times \left\{ i_{k+2} \right\} \times \cdots \times \left\{ i_K \right\} \; | \; i_1 \in [n_1], \ldots, i_K \in [n_K]
 \right\}.
 $$
\normalsize
Every element of $\mathcal{S}^{(k)}$ is a set of indices, and we can identify a tensor $\vA \in \R^{n_1 \times \cdots \times n_{k-1} \times n_{k+2} \times \cdots \times n_{K}}$ with a map $\mathcal A: \mathcal{S}^{(k)} \rightarrow \R$. Using this identification, every element of $\vA$ can thus also be referenced by $S \in \mathcal{S}^{(k)}$. To keep our notation succinct, we will thus refer to $\vA_S$ as the element corresponding to $S$ under this identification. Thus if $S=\left\{ i_1 \right\} \times \cdots  \times \left\{ i_{k-1} \right\} \times [n_k] \times [n_{k+1}] \times \left\{ i_{k+2} \right\} \times \cdots \times \left\{ i_K \right\}$, the element:
$$
\vA_S = \vA_{i_1, \ldots, i_{k-1}, i_{k+2}, \ldots, i_K}.
$$

Using this notation, we can define a high-order contraction. A mode-$k$ contraction of $\vX$ with respect to a tensor $\vA$ is thus:
\begin{equation} \label{eq:contraction_ho}
X_{\vA}^{k} := \sum_{S \in \mathcal{S}^{(k)}}\vA_S \vX_S. 
\end{equation}
Note that since $X_{\vA}^{k}$ is a sum of (two-dimensional) slices, it is a matrix. As in the third order case, we will be interested in contractions where $\vA$ is either random or a coordinate tensor. The analogue of Lemma \ref{lemma:fund_lemma} for the higher order case is the following:
\begin{lemma} \label{lemma:decomp_high}
Let $\vX$ have the decomposition $\vX = \sum_{l=1}^r \bigotimes_{p=1}^{K} u^p_l$. Then we have that the contraction $X^k_{\vA}$ has the following matrix decomposition:
\begin{equation} \label{eq:decomp_high1}
X^k_{\vA}=\sum_{l=1}^r \nu_l^k  u_l^k \left( u_l^{k+1} \right)^{T},
\end{equation}
where  $\nu_l^k:= \langle \vA, \underset{{p \neq k, k+1}}{\bigotimes}u^p_l\rangle $. Furthermore, if $X_{\vB}^k$ is another contraction with respect to $\vB$, then the eigenvectors of the matrices
\begin{align} \label{eq:m_and_n}
M_1=X^k_{\vA} \left( X^k_{\vB} \right) ^{\dag} 
& \qquad M_2=\left( \left( X^k_{\vB} \right)^{\dag} X^k_{\vA} \right)^{T}
\end{align}
respectively are $\{ u^k_l \}_{l=1, \ldots, r}$ and $\{ u^{k+1}_l \}_{l=1, \ldots, r}$.
\end{lemma}
\begin{proof}
It is straightforward to verify by simply expanding the definition of $X_{\vA}^k$  using the definition of contraction \eqref{eq:contraction_ho}:
\begin{align*}
\left[ X_{\vA}^{k} \right]_{j_{k},j_{k+1}}
={\sum_{j_1, \ldots, j_{k-1}, j_{k+2}, \ldots, j_K} \sum_{l=1}^{r} \left( \prod_{p=1}^{K} \left( u_l^{p} \right)_{j_k} \right) \vA_{j_1, \ldots, j_{k-1},{j_{k+2}}, \ldots, j_{K}}}.
\end{align*} \normalsize
Rearranging terms, we get the decomposition \eqref{eq:decomp_high1}. The eigenvalues of $M_1, M_2$ follow along similar lines to the proof of Lemma \ref{lemma:fund_lemma}.
\end{proof}
As a consequence of the above lemma, if $\vX$ is of low rank, so are all the contractions. The notions of non-degeneracy and pairwise genericity of contractions extend in a natural way to the higher order case. We say that a contraction $X^k_{\vA}$ is non-degenerate if 
$\nu_l^k \neq 0$ for all $l=1, \ldots, r$. Furthermore, a pair of contractions is pairwise generic if the corresponding ratios $\nu_l^k$ are all distinct for $l=1, \ldots, r$. Non-degeneracy and pairwise genericity hold almost surely when the contractions are computed with random tensors $\vA$, $\vB$ from appropriate random ensembles (e.g. $i.i.d.$ normally distributed entries).
In much the same way as the third order case, Leurgans' algorithm can be used to perform decomposition of low-rank tensors using Lemma \ref{lemma:decomp_high}. This is described in Algorithm \ref{alg:leurgans_high}.

\begin{algorithm}[!ht]
   \caption{Leurgans' Algorithm for Higher Order Tensors}
   \label{alg:leurgans_high}
\begin{algorithmic}[1]
   \STATE {\bfseries Input:} Tensor $\vX$. 
   %\REPEAT 
\FOR{ $k=1$ to $K-1$}
\STATE Compute contractions $X_{\vA}^k$ and $X_{\vB}^k$ for some tensors $\vA$ and $\vB$ of appropriate dimensions, such that the contractions are non-degenerate and pairwise generic.
\STATE Compute eigen-decompositions of $M_1 := X_{\vA}^k \left(X_{\vB}^k \right)^{\dagger}$ and $M_2 := \left(X_{\vB}^k\right)^{\dagger}X_{\vA}^k$. Let $\tilde{U}^k$ and $\tilde{U}^{k+1}$ denote the matrices whose columns are the eigenvectors of $M_1$ and $M_2^T$ respectively corresponding to the non-zero eigenvalues, in sorted order. (Let $r$ be the (common) rank of $M_1$ and $M_2$.) 
\STATE If $k=1$, let $U^1:=\tilde{U}^1$ and $U^2:= \tilde{U}^2$. 

\STATE If $k \geq 2$, simultaneously reorder the columns of $\tilde{U}^k$, $\tilde{U}^{k+1}$, also performing simultaneous sign reversals as necessary so that the columns of $\tilde{U}^{k}$ obtained match with the columns of  $U^{k}$ (obtained in the previous iteration), call the resulting matrices $U^{k}$, $U^{k+1}$. (The eigenvectors corresponding to mode $k+1$, thus obtained are denoted as $\{ u_l^{k+1} \}_{l=1, \ldots, r}$.)

\ENDFOR
\STATE Solve for $\lambda_l$ in the (over-determined) linear system $$\vX=  \sum_{l=1}^{r} \lambda_l  \bigotimes_{k=1}^{K} u^k_l. $$

  \STATE {\bfseries Output:} Recovered tensor $\vX=\sum_{l=1}^r  \lambda_l \bigotimes_{k=1}^{K} u_l^k$.
\end{algorithmic}
\end{algorithm}

Finally, the notion of \emph{separable measurements} can be extended to higher order tensors in a natural way.

\begin{definition}
Consider a linear operator $\mathcal{L}: \R^{n_1 \times  \cdots  \times n_K} \rightarrow \R^n$. We say that $\mathcal{L}$ is separable with respect to the $k^{th}$ mode if there exist $\vW \in \R^{n_1 \times \cdots \times n_{k-1} \times n_{k+2} \times \cdots \times n_K}$ and a linear operator $\mathcal{T}^{(k)}: \R^{n_k \times n_{k+1} } \rightarrow \R^n$, such that for every $\vX \in \R^{n_1 \times \cdots \times n_K}$: 
$$
\mathcal{L}\left( \vX \right) = \sum_{S \in \mathcal{S}^{(k)}} \vW_S \,\mathcal{T}^{(k)} \left( X^k_S \right).
$$
\end{definition}
Analogous to the third order case, we assume that we are presented with two sets of separable measurements per mode:
\begin{align*}
y_1^{(k)}&=\mathcal{L}^{(k)}_1\left( \vX \right) = \sum_{S \in \mathcal{S}^{(k)}} \left(\vW_1\right)_{S} \,\mathcal{T}^{(k)} \left( X^k_S \right)\\
  y_2^{(k)} & =\mathcal{L}^{(k)}_2\left( \vX \right) = \sum_{S \in \mathcal{S}^{(k)}} \left( \vW_2\right)_{S} \,\mathcal{T}^{(k)} \left( X^k_S \right)
\end{align*}
for $k=1, \ldots, K-1$ with each of $y_1^{(k)}, y_2^{(k)} \in \R^{m_{k}}$. Once again, by separability we have:
$$
y_1^{(k)} = \mathcal{T}^{(k)} \left( X_{{\vW}_1}^k \right) \qquad y_2^{(k)} = \mathcal{T}^{(k)} \left( X_{{\vW}_2}^k \right),
$$
and since the contractions $X_{{\vW}_1}^k$ and $ X_{{\vW}_2}^k$ are low rank, nuclear norm minimization can be used to recover these contractions via:

\begin{equation} \label{eq:opt_ho1}
\underset{Z_1}{\text{minimize}}  \qquad \|Z_1 \|_* \qquad \text{subject to}  \qquad  y_1^{(k)}=\mathcal{T}^{(k)}\left( Z_1 \right), 
\end{equation}

\begin{equation} \label{eq:opt_ho2}
\underset{Z_2}{\text{minimize}}  \qquad \|Z_2 \|_* \qquad \text{subject to}  \qquad  y_2^{(k)}=\mathcal{T}^{(k)}\left( Z_2 \right), 
\end{equation}
for each $k=1, \ldots, K-1$. After recovering the two contractions for each mode, we can then apply (the higher order) Leurgans' algorithm to recover the tensor factors. The precise algorithm is described in Algorithm \ref{alg:trecs_high}. Provided the $\mathcal{T}^{(k)}\left( \cdot \right)$ are faithful, the tensor contractions can be successfully recovered via nuclear norm minimization. Furthermore, if the contractions are non-degenerate and pairwise generic, the method can successfully recover the entire tensor.

\begin{algorithm}[!ht]
   \caption{T-ReCs for Higher Order Tensors}
   \label{alg:trecs_high}
\begin{algorithmic}[1]
   \STATE {\bfseries Input:} Measurements $y_i^{(k)} = \mathcal{L}_i^{(k)} \left( \vX \right)$, for $k=1, \ldots, K$, $i=1, 2.$ 
   %\REPEAT 
\FOR{ $k=1$ to $K-1$}
\STATE Solve convex optimization problems \eqref{eq:opt_ho1} and \eqref{eq:opt_ho2} to obtain optimal solutions $Z_1^*$ and $Z_2^*$ respectively. 

\STATE Compute eigen-decompositions of $M_1 := Z_1^{*}(Z_2^*)^{\dagger}$ and $M_2 := (Z_2^*)^{\dagger}Z_1^*$. Let $\tilde{U}^k$ and $\tilde{U}^{k+1}$ denote the matrices whose columns are the normalized eigenvectors of $M_1$ and $M_2^T$ respectively corresponding to the non-zero eigenvalues, in sorted order. (Let $r$ be the (common) rank of $M_1$ and $M_2^{T}$.) 

\STATE If $k=1$, let $U^1:=\tilde{U}^1$ and $U^2:= \tilde{U}^2$. 

\STATE If $k \geq 2$, simultaneously reorder the columns of $\tilde{U}^k$, $\tilde{U}^{k+1}$, also performing simultaneous sign reversals as necessary so that the columns of $\tilde{U}^{k}$ obtained match with the columns of  $U^{k}$ (obtained in the previous iteration), call the resulting matrices $U^{k}$, $U^{k+1}$. (The eigenvectors corresponding to mode $k+1$, thus obtained are denoted as $\{ u_l^{k+1} \}_{l=1, \ldots, r}$.)
\ENDFOR
\STATE Solve for $\lambda_l$ in the (over-determined) linear system $$y_i^{(k)}=  \sum_{l=1}^{r} \lambda_l \mathcal{L}_i^{(k)}\left(\bigotimes_{k=1}^{K} u^k_l \right), \; \; k=1, \ldots, K-1, \; i=1,2.$$

  \STATE {\bfseries Output:} Recovered tensor $\vX=\sum_{l=1}^r  \lambda_l \bigotimes_{k=1}^{p} u_l^k$.
\end{algorithmic}
\end{algorithm}

\subsection{Separable Random Projections}

%Consider a tensor $\vX \in \R^{n_{1} \times \cdots \times n_{K}}$ of order $K$ and dimensions $n_1\times \cdots \times n_{K}$. Without loss of generality we assume $n_1 \leq \ldots \leq n_K$. Let the rank of this tensor be $r \leq n_1$ and be given by the decomposition:
%\begin{align*}
%\vX &= \sum_{l=1}^r u_l^1 \otimes \ldots \otimes u_l^K = \sum_{l=1}^r \bigotimes_{p=1}^{K} u^p_l,
%\end{align*}
%where $u_l^p \in \R^{n_p}$.
%
%A word on notation: in this subsection we reserve $l$ as the index that varies over the $r$ rank-one components of the tensor in its decomposition and $k$ and $p$ as the indices that vary over the different modes of the tensor.

%A word on notation: in this subsection we reserve the index $i$ for measurements, $l$ as the index that varies over the $r$ rank-one components of the tensor in its decomposition and $k$ as the index that varies over the different modes of the tensor.
%
%Much like the order $3$ case, we will design random linear sensing operators $\mathcal{L}$, and show that the underlying ill-posed tensor inverse problem can be solved exactly from an order-optimal number of measurements. The key idea in the method will be similar to the idea driving Algorithm \ref{alg:trecs}; we successively recover pairs of factors in each round of the algorithm (e.g. $\left\{ u_l^1 \right\}_{l=1, \ldots, r}$ and $\left\{ u_l^2 \right\}_{l=1, \ldots, r}$ , then in the next round $\left\{ u_l^2 \right\}_{l=1, \ldots, r}$ and $\left\{ u_l^3 \right\}_{l=1, \ldots, r}$, etc.). In the end we then recover the required scale factors.

Given tensors $\vA \in \R^{n_1 \times \cdots \times n_{K_1}}$, $\vB \in \R^{n_{K_1+1} \times  \cdots  \times n_{K_1+K_2}}$,  $\vC \in \R^{n_{K_1+K_2+1} \times \cdots \times n_{K_1+K_2+K_3}}$ of orders $K_1$, $K_2$ and $K_3$  respectively with $K_1+K_2+K_3=K$, we define their outer product as:
\begin{align*}
\left[ \vA \otimes \vB \otimes \vC \right]_{i_1, \ldots, i_K} :=  \left[ \vA \right]_{i_1, \ldots, i_{K_{1}}} \left[ \vB \right]_{i_{K_{1}+1}, \ldots, i_{K_{1}+K_{2}}} \left[ \vC \right]_{i_{K_{1}+K_{2}+1}, \ldots, i_{K_{1}+K_{2}+K_{3}}} 
\end{align*}
Note also that the inner-product for higher order tensors is defined in the natural way:
$$
\langle\vT, \vX \rangle := \sum_{i_1, \ldots, i_K} \left[ \vT \right]_{i_1, \ldots, i_K} \left[ \vX \right]_{i_1, \ldots, i_K}.
$$
%
%For the higher order case, we need a more refined notion of contractions. Given tensors $\vX \in \R^{n_{1} \times  \ldots  \times n_{K}}$, $\vA \in \R^{n_{1} \times  \ldots \times n_{k-1}}$, $\vB \in \R^{n_{k+2} \times  \ldots \times n_{K}}$ we define the contraction of $X$ into the $(k,k+1)$ indices with respect to $\vA, \vB$, denoted by $X^k_{\vA, \vB} \in \R^{n_k \times n_{k+1}}$ by
%\begin{align} \label{eq:contr_higher}
%&\left[ X_{\vA, \vB}^{k} \right]_{j_{k},j_{k+1}} \\ 
%\notag
%&=\sum_{j_1, \ldots, j_{k-1}, j_{k+2}, \ldots, j_K} \vX_{j_1, \ldots, j_K} \vA_{j_1, \ldots, j_{k-1}} \vB_{j_{k+2}, \ldots, j_{K}}.
%\end{align}
%As $k$ ranges from $1, \ldots, K-1$, we obtain contractions into different dimensions of the tensor (note of course that the dimensions of $\vA$ and $\vB$ need to change in a commensurate manner). These contractions will play a similar critical role in our algorithm for the higher order case.

In this higher order setting, we also work with specific separable random projection operators, which are defined as below:
\begin{equation} \label{eq:meas} 
\begin{split}
y_1^{(k)}=\mathcal{L}_1^{(k)} \left( \vX \right) := \left[ \begin{array}{c}
\langle \vA_k \otimes \Gamma^{(k)}_1 \otimes \vB_k, \vX \rangle \\
\vdots \\
\langle \vA_k \otimes \Gamma^{(k)}_{m_k} \otimes \vB_k, \vX \rangle
\end{array}
\right] \\
y_2^{(k)}=\mathcal{L}_2^{(k)} \left( \vX \right) := \left[ \begin{array}{c}
\langle \vC_k \otimes \Gamma^{(k)}_1 \otimes \vD_k, \vX \rangle \\
\vdots \\
\langle \vC_k \otimes \Gamma^{(k)}_{m_k} \otimes \vD_k, \vX \rangle
\end{array}
\right]
\end{split}
\end{equation} 

 \small
\begin{equation} \label{eq:meas} 
%\begin{split}
%y_i^{(k)}=\mathcal{L}_i^{(k)} \left( \vX \right) := \langle \vA_k \otimes \Gamma^{(k)}_i \otimes \vB_k, \vX \rangle \qquad & i=1, \ldots, m_k\\
%y_i^{(k)}=\mathcal{L}_i^{(k)} \left( \vX \right) := \langle \vC_k \otimes \Gamma^{(k)}_i \otimes \vD_k, \vX \rangle \qquad & i=m_k+1, \ldots, 2m_k.
%\end{split}
\end{equation} \normalsize
In the above expressions, $\vA_k, \vC_k \in \R^{n_{1} \times \cdots \times n_{k-1}}$, and $\vB_k, \vD_k \in \R^{n_{k+2} \times \cdots \times n_{K}}$. The tensors $\vA_k$, $\vB_k$, $\vC_k$, $\vD_k$ are all chosen so that their entries are randomly and independently distributed according to $\mathcal{N}(0,1)$ and subsequently normalized to have unit Euclidean norm. The matrices $\Gamma^{(k)}_i \in \R^{n_k \times n_{k+1}}$ for $i=1,\ldots, m_k$ have entries randomly and independently distributed according to $\mathcal{N}(0,1)$.
For each $k$ we have $2m_k$ measurements so that in total there are $2\sum_{i=1}^{K-1} m_k$ measurements.

\begin{lemma} \label{lemma:high_eq}
We have the following identity:
$$
\langle  \vA_k \otimes \Gamma^{(k)}_i \otimes \vB_k, \vX \rangle = \langle \Gamma^{(k)}_i, X^k_{\vA_k \otimes \vB_k}\rangle.
$$
\end{lemma}
\begin{proof}
The proof is analogous to that of Lemma \ref{lemma:main_obs}.
\begin{align*}
&\langle \vA_k \otimes \Gamma^{(k)}_i \otimes \vB_k, \vX\rangle \\
& = \sum_{l=1}^{r} \langle \vA_k \otimes \Gamma^{(k)}_i \otimes \vB_k, \bigotimes_{p=1}^{K} u^p_l\rangle\\
&\stackrel{(\text{i})}{=}\sum_{l=1}^r \langle \vA_k, \bigotimes_{p=1}^{k-1}u^p_l\rangle  \langle \vB_k, \bigotimes_{p=k+2}^{K}u^p_l\rangle \langle \Gamma^{(k)}_i ,u_l^{k} \otimes u_l^{k+1} \rangle \\
&= \sum_{l=1}^{r}  \langle \vA_k \otimes \vB_k, \bigotimes_{p=1}^{k-1}u^p_l \bigotimes_{p=k+2}^{K}u^p_l\rangle  \langle \Gamma^{(k)}_i ,u_l^{k} \otimes u_l^{k+1} \rangle  \\
&\stackrel{(\text{ii})}{=} \sum_{l=1}^r \nu_l^{k} \langle \Gamma^{(k)}_i ,u_l^{k} \otimes u_l^{k+1} \rangle \; \; \; (\text{where } \nu_l^k= \langle \vA_k \otimes \vB_k, \underset{{p \neq, k, k+1}}{\bigotimes}u^p_l\rangle )\\
&= \langle \Gamma^{(k)}_i, \sum_{l=1}^{r} \nu^k_l u_l^{k} \otimes u_l^{k+1} \rangle \\
&= \langle \Gamma^{(k)}_i, X^k_{\vA_k \otimes \vB_k} \rangle.
\end{align*}
The equality (i) follows from the identity $\langle a \otimes b \otimes c , x \otimes y \otimes z \rangle = \langle a, x \rangle \langle b, y \rangle \langle c, z \rangle$ for $a, b, c, x, y, z$ of commensurate dimensions. The equality (ii) follows from the definition $\nu_k^l$ in Lemma \ref{lemma:decomp_high}
\end{proof}
It follows immediately from Lemma \ref{lemma:high_eq} that in \eqref{eq:meas}, for each $k=1, \ldots, K-1$, $\mathcal{L}_1^{(k)} \left( \cdot \right)$, $\mathcal{L}_2^{(k)} \left( \cdot \right)$ are in fact, separable so that Algorithm \ref{alg:trecs_high} is applicable. Recovering the contractions involves solving a set of nuclear norm minimization sub-problems for each $k=1, \ldots, K-1$:
\begin{equation} \label{eq:opt_high1}
\begin{split}
\underset{Z_1}{\text{minimize}} &\qquad  \|Z_1\|_* \\
 \text{subject to } & \qquad  y_1^{(k)}= \left[ \begin{array}{c}
\langle  \Gamma^{(k)}_1, Z_1 \rangle \\
\vdots \\
\langle \Gamma^{(k)}_{m_k}, Z_1 \rangle
\end{array}
\right] 
 \end{split}
\end{equation}

\begin{equation} \label{eq:opt_high2}
\begin{split}
\underset{Z_2}{\text{minimize}} &\qquad  \|Z_2\|_* \\
 \text{subject to } & \qquad  y_2^{(k)}= \left[ \begin{array}{c}
\langle  \Gamma^{(k)}_1, Z_2 \rangle \\
\vdots \\
\langle \Gamma^{(k)}_{m_k}, Z_2 \rangle
\end{array}
\right] 
 \end{split}
\end{equation}

We have the following lemma concerning the solutions of these optimization problems:
\begin{lemma}
Suppose $m_k > 3r(n_k+n_{k+1}-r)$. Then the unique solutions to problems \eqref{eq:opt_high1} and \eqref{eq:opt_high2} are $X^{k}_{\vA_{k} \otimes\vB_k}$ and $X^{k}_{\vC_{k}\otimes \vD_k}$ respectively with high probability. 
\end{lemma}
The proof is analogous to that of Lemma \ref{lemma:contr_rec}.

We have the following theorem concerning the performance of Algorithm \ref{alg:trecs_high}.
\begin{theorem} \label{thm:exact_rec_high}
Let $\vX \in \R^{n_1\times \cdots \times n_K }$ be an unknown tensor of interest with rank $r \leq \min\left\{n_1, \ldots, n_K\right\}$. Suppose $m_k > 3r(n_k+n_{k+1}-r)$ for each $k=1, \ldots, K-1$. Then the procedure outlined in Algorithm \ref{alg:trecs_high} succeeds in exactly recovering $\vX$ and its low rank decomposition with high probability.
\end{theorem}
The proof parallels that of the proof of Theorem \ref{thm:exact_rec_sense} and is omitted for the sake of brevity.
\begin{remarks} \hspace{10mm}
\begin{itemize}
\item Note that the overall sample complexity is $2\sum_{k=1}^{K} m_k$, i.e., $6\sum_{k=1}^{K-1} r(n_{k}+n_{k+1}-r)$. This constitutes an order optimal sample complexity because a tensor of rank $r$ and order $K$ of these dimensions has $r\sum_{k=1}^{K}n_k$ degrees of freedom. In particular, when the tensor is ``square'' i.e.,  $n_1= \cdots = n_K=n$ the number of degrees of freedom is $Knr$ whereas the achieved sample complexity is no larger than $12Knr$, i.e., $O(Knr)$.
\item As with the third order case, the algorithm is tractable. The main operations involve solving a set of  matrix nuclear norm minimization problems (i.e. convex programs), computing eigenvectors, and aligning them. All of these are routine, efficiently solvable steps (thus ``polynomial time'') and indeed enables our algorithm to be scalable.

\end{itemize}
\end{remarks}
\subsection{Tensor Completion} \label{sec:higher_order_completion}
The method described in Section \ref{sec:third_order} for tensor completion can also be extended to higher order tensors in a straightforward way. Consider a tensor $\vX \in \R^{n_{1} \times \cdots \times n_{K}}$ of order $K$ and dimensions $n_1\times \cdots \times n_{K}$. Let the rank of this tensor be $r \leq \min \left\{ n_1, \ldots, n_K \right\}$ and be given by the decomposition:
\begin{align*}
\vX &= \sum_{l=1}^r u_l^1 \otimes \ldots \otimes u_l^K = \sum_{l=1}^r \bigotimes_{p=1}^{K} u^p_l,
\end{align*}
where $u_l^p \in \R^{n_p}$.
%We will be interested in slices of the given tensor that are identified by picking two \emph{consecutive} modes $(k, k+1)$, and by fixing all the indices not in those modes, i.e. $i_1 \in [n_{1}], \ldots, i_{k-1} \in [n_{k-1}], i_{k+2} \in [n_{k+2}], \ldots, i_K \in [n_K]$. Thus the indices of a slice $S^{(k)}$ are:
%\begin{align*}
%S^{(k)}:=\left\{ i_1 \right\} \times \ldots  \left\{ i_{k-1} \right\} \times [n_k] \times [n_{k+1}] \times \left\{ i_{k+2} \right\} \times \cdots \times \left\{ i_K \right\},
%%S^{(k)}:=\left\{ (i_1, \ldots, i_{k-1}, s,t, i_{k+2}, \ldots, i_K) : \right. \\
%%\left. s \in [n_k], t \in [n_{k+1}]  \right\}.
% \end{align*}
% and the corresponding slice may be viewed as a matrix, denoted by $X^{k}$, (to simplify notation, we will hide the indices $i_1, \ldots, i_K$)
% 
%While slices of tensors can be defined more generally (i.e. the modes need not be consecutive), in this paper we will only need to deal with such ``contiguous'' slices. 
Extending the sampling notation for tensor completion from the third order case, we define $\Omega$ to be the set of indices corresponding to the observed entries of the unknown low rank tensor $\vX$, and define:
\begin{align*}
\Omega^{(k)}:=S^{(k)} \cap \Omega \qquad m^{(k)}:=|\Omega^{(k)}|,
\end{align*}
where ${S}^{(k)} \in \mathcal{S}^{k}$.
Akin to the third-order case, along each pair of consecutive modes, we will need samples from two distinguished slices. We  denote the index set of these distinct slices by $S_1^{(k)}$ and $S_2^{(k)}$, the corresponding slices by $X_1^k$ and $X_2^k$, the index set of the samples revealed from these slices by $\Omega_1^{(k)}$ and $\Omega_2^{(k)}$, and their cardinality by $m_1^{(k)}$ and $m_2^{(k)}$.

It is a straightforward exercise to argue that observations obtained from each slice $S_i^{(k)}$, $i=1,2$ correspond to separable measurements, so that Algorithm \ref{alg:trecs_high} applies.
The first step of Algorithm \ref{alg:trecs_high} involves solving a set of nuclear norm minimization sub-problems (two problems for each $k=1, \ldots, K-1$) to recover the slices:
\begin{equation} \label{eq:opt_high_comp1}
\underset{Z_1}{\text{minimize}} \qquad \|Z_1\|_* \qquad \text{subject to } \qquad  \vX_{\Omega_1^{(k)}}=\left[ Z_1 \right]_{\Omega_1^{(k)}}
\end{equation}
\begin{equation} \label{eq:opt_high_comp2}
\underset{Z_2}{\text{minimize}} \qquad  \|Z_2\|_* \qquad \text{subject to } \qquad  \vX_{\Omega_2^{(k)}}=\left[ Z_2 \right]_{\Omega_2^{(k)}}
\end{equation}
%\begin{equation} \label{eq:opt_high2}
%\begin{split}
%\underset{W_k}{\text{minimize}} \qquad & \|W_k\|_* \\
%\text{subject to } \qquad & y_i^{(k)}= \langle \Gamma_k^{(i)}, W_k \rangle, \qquad i=m_k+1, \ldots, 2m_k
%\end{split}
%\end{equation}
%We have the following lemma concerning the solutions of these optimization problems:
%\begin{lemma}
Each of these optimization problems will succeed in recovering the corresponding slices provided incoherence, non-degeneracy and the slice conditions hold (note that these notions all extend to the higher order setting in a transparent manner). Under these assumptions, if the entries from each slice are uniformly randomly sampled with cardinality at least
$m^{(i)}_k > C(n_k+n_{k+1})\log^2(n_{k+1})$, $i=1,2$ for some constant $C$, then the unique solutions to problems \eqref{eq:opt_high1} and \eqref{eq:opt_high2} will be $X^{k}_{1}$ and $X^{k}_2$ respectively with high probability.  
%\end{lemma}
%The proof is analogous to that of Lemma \ref{lemma:contr_rec}.

Once the slices $X^{k}_{1}$ and $X^{k}_{2}$ are recovered correctly using \eqref{eq:opt_high1}, \eqref{eq:opt_high2} for each $k=1, \ldots, K-1$ one can compute $M_1:=X^{k}_{1} \left( X^{k}_{2} \right)^{\dagger}$ and $M_2:=\left( X^{k}_{2}\right)^{\dagger}X^{k}_{1} $ and perform eigen-decompositions to obtain the factors (up to possible rescaling) $\{u_l^k \}$ and $\{u_l^{k+1} \}$. Finally, once the tensor factors are recovered, the tensor itself can be recovered exactly by solving a system of linear equations. These observations can be summarized by the following theorem:
\begin{theorem} \label{thm:exact_rec_highcomp}
Let $\vX \in \R^{n_1\times \cdots \times n_K }$ be an unknown tensor of interest with rank $r \leq \min\left\{n_1, \ldots, n_K\right\}$. Suppose we obtain $m_1^{(k)}$ and $m_2^{(k)}$ random samples from each of the two distinct mode $k$ slices for each $k=1, \ldots, K-1$. Furthermore suppose the tensor $X$ is incoherent, satisfies the slice conditions for each mode, and the slices from which samples are obtained satisfy non-degeneracy and pairwise genericity for each mode. Then there exists a constant $C$ such that if
$$m_i^{(k)}> C(n_k+n_{k+1})\log^2(n_{k+1}) \qquad i \in \left\{1,2 \right\}$$
the procedure outlined in Algorithm \ref{alg:trecs_high} succeeds in exactly recovering $\vX$ and its low rank decomposition with high probability.
\end{theorem}
We finally remark that the resulting sample complexity of the entire algorithm is $\sum_{k=1}^{K-1}(m^{(k)}_1~+~m^{(k)}_2)$, which is $O(Kr n_K  \log^2(n_K))$.

\section{Experiments}
In this section we present numerical evidence in support of our algorithm. We conduct experiments involving (suitably) random low-rank target tensors and their recovery from (a) Separable Random Projections and (b) Tensor Completion. We obtain phase transition plots for the same, and compare our performance to that obtained from the matrix-unfolding based approach proposed in \cite{squaredeal}.  For the phase transition plots, we implemented matrix completion using the method proposed in \cite{rao2013conditional}, since the SDP approach for exact matrix completion of unfolded tensors was found to be impractical for even moderate-sized problems. 
\label{sec:exp}

\subsection{Separable Random Projections : Phase Transition}
In this section, we run experiments comparing T-ReCs to tensor recovery methods based on ``matricizing" the tensor via unfolding \cite{squaredeal}. 
%\twocolumn
\begin{figure}%[!ht]
\centering
\subfigure[Tensor recovery using T-ReCs. (n=$30$)]{
\includegraphics[trim = 40mm 70mm 40mm 70mm, clip = true, scale = 0.5]{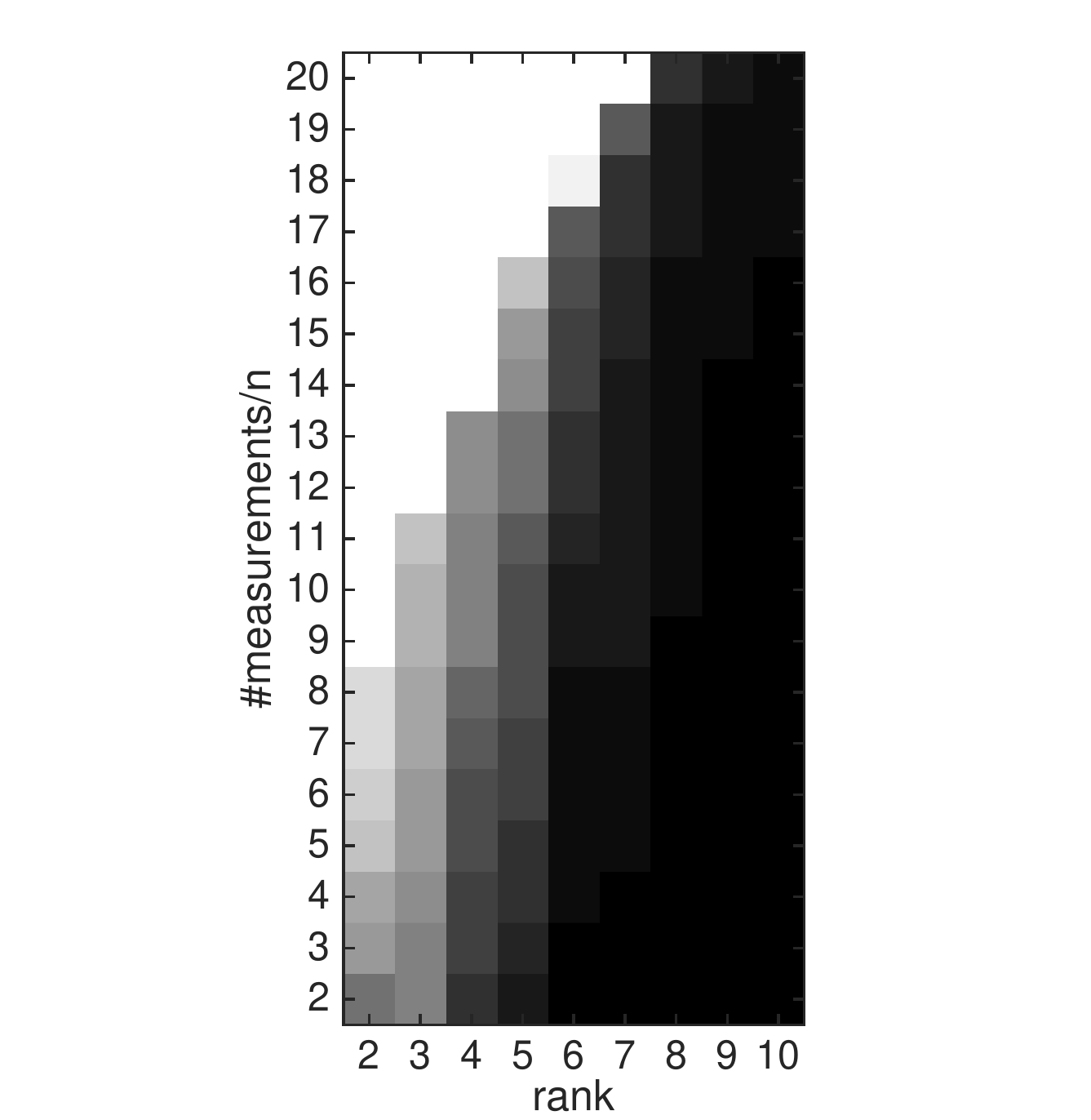}
}
\subfigure[Tensor recovery using \cite{squaredeal}. (n=$30$)]{
\includegraphics[trim = 40mm 70mm 40mm 70mm, clip = true, scale = 0.5]{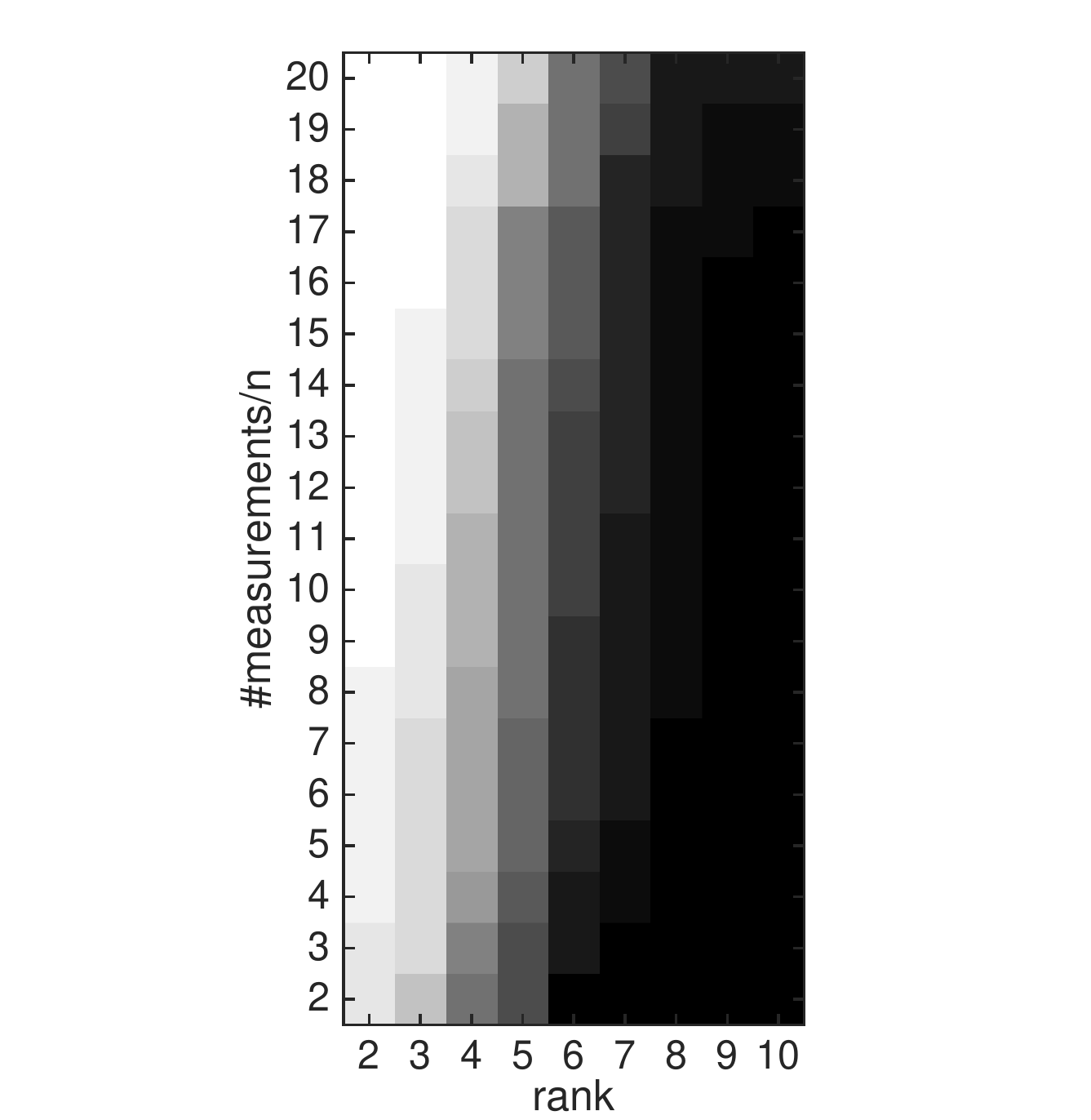}
}
\caption{Phase transition diagram for tensor recovery using our method. White indicates a probability of recovery of 1, while black indicates failure of exact recovery. Note that in the matrix unfolding case, one requires more measurements compared to our method to achieve the same probability of recovery for a given rank. }
\label{phase}
\end{figure} 
We consider a tensor of size $30 \times 30 \times 30$ whose factors $U, V, W \in \R^{n \times r}$ are i.i.d standard Gaussian entries. We vary the rank $r$ from 2 to 10, and look to recover these tensors from different number of measurements $m \in [2,20]*n$. For each $(r,n)$ pair, we repeated the experiment 10 times, and consider recovery a ``success" if the MSE is less than $10^{-5}$. Figure \ref{phase} shows that the number of measurements needed for accurate tensor recovery is typically less in our method, compared to the ones where the entire tensor is converted to a matrix for low rank recovery.

\subsection{Tensor Completion: Phase Transition}
\label{sec:exp_completion}

 We again considered tensors of size $30 \times 30 \times 30$,  varied the rank of the tensors from $2$ to $10$, and obtained random measurements from four slices (without loss of generality we may assume they are the first 2 slices across modes 1 and 2). The number of measurements obtained varied as $n \times [2,20]$. Figure \ref{ten_rec} shows the phase transition plots of our method. We deem the method to be a ``success" if the MSE of the recovered tensor is less than $10^{-5}$. Results were averaged over $10$ independent trials. 
 %\textcolor{red}{GT: 5 or 10 independent trials?}

\begin{figure}[!h]
\centering
\subfigure[Phase transition for tensor completion using T-ReCs. (n=$30$)]{
\includegraphics[trim = 60mm 70mm 60mm 70mm, clip = true, scale = 0.5]{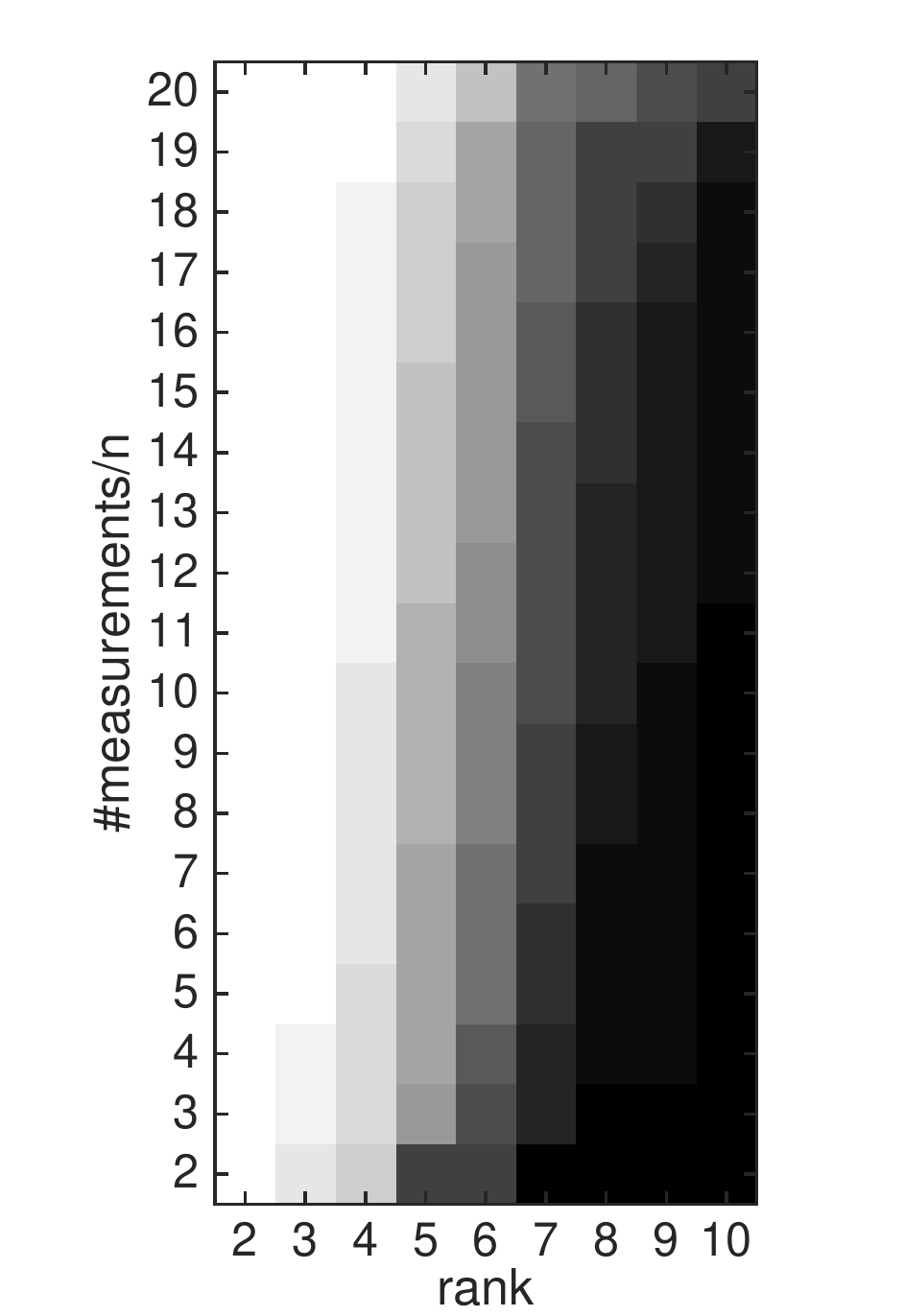}
}
\subfigure[Phase transition for tensor completion using \cite{squaredeal}. (n=$30$)]{
\includegraphics[trim = 60mm 70mm 60mm 70mm, clip = true, scale = 0.5]{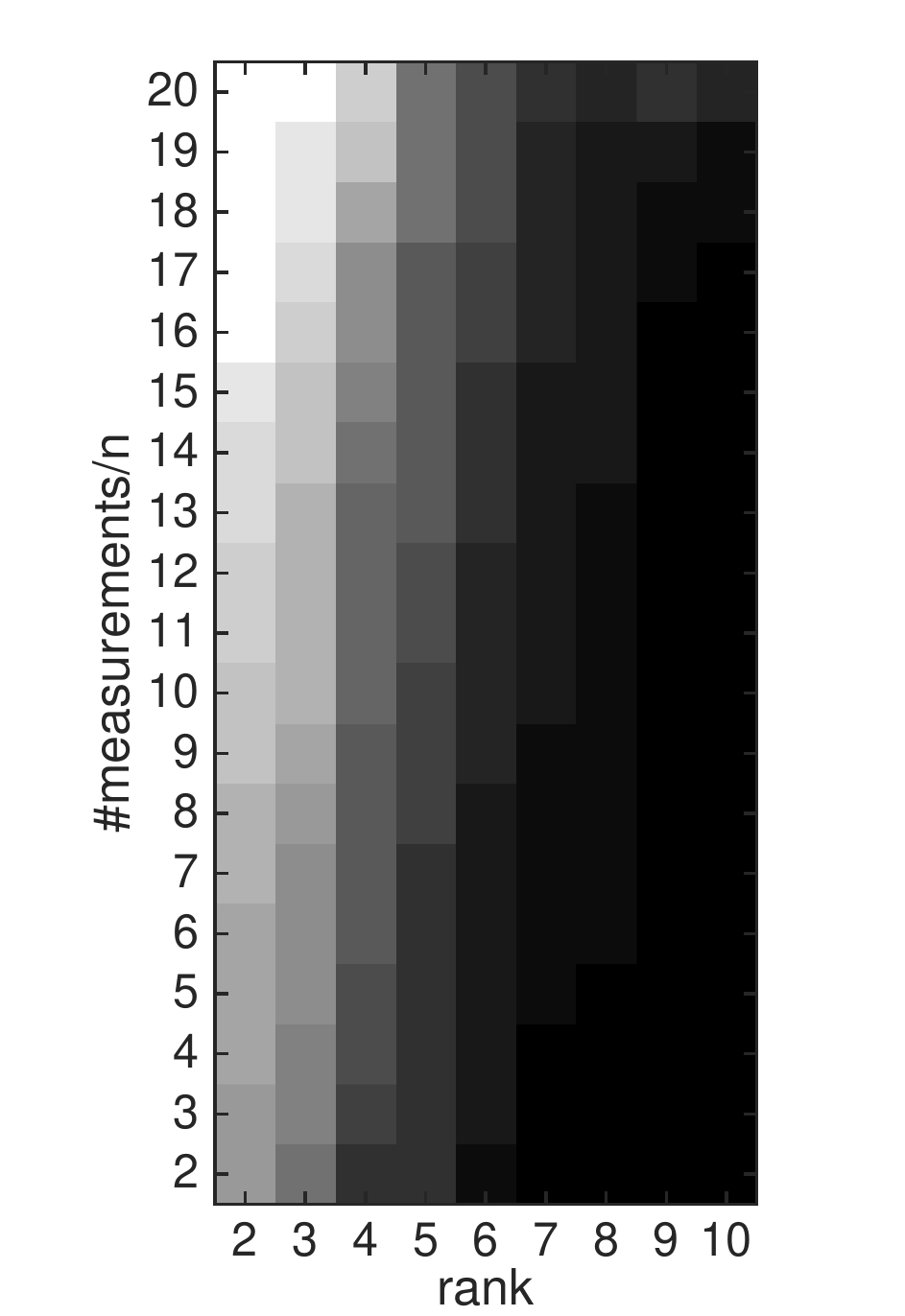}
}
\label{ten_rec}
\caption{Phase transition plots for tensor recovery. Results are averaged over 10 independent trials. White indicates success whereas black indicates failure.}
% \textcolor{red}{GT: 5 or 10 independent trials? And which one is based on T-ReCs?}} 
\end{figure}

\subsection{Speed Comparisons}
\label{sec:speed}

We finally compared the time taken to recover an $n \times n \times n$ tensor of rank 3. Figure \ref{time} shows that, T-ReCs with four smaller nuclear norm minimizations is far more scalable computationally as compared to the method of unfolding the tensor to a large matrix and then solving a single nuclear norm minimization program. This follows since matricizing the tensor involves solving for an $n^2 \times n$ matrix. Our method can thus be used for tensors that are orders of magnitude larger than competing methods. 

\begin{figure}[!h]
\begin{center}
\subfigure[Time taken to recover third order tensors. The numbers $5$ and $10$ in the legend refer to the cases where we obtain $5n$ and $10n$ measurements respectively]{
\includegraphics[width = 60mm, height = 50mm]{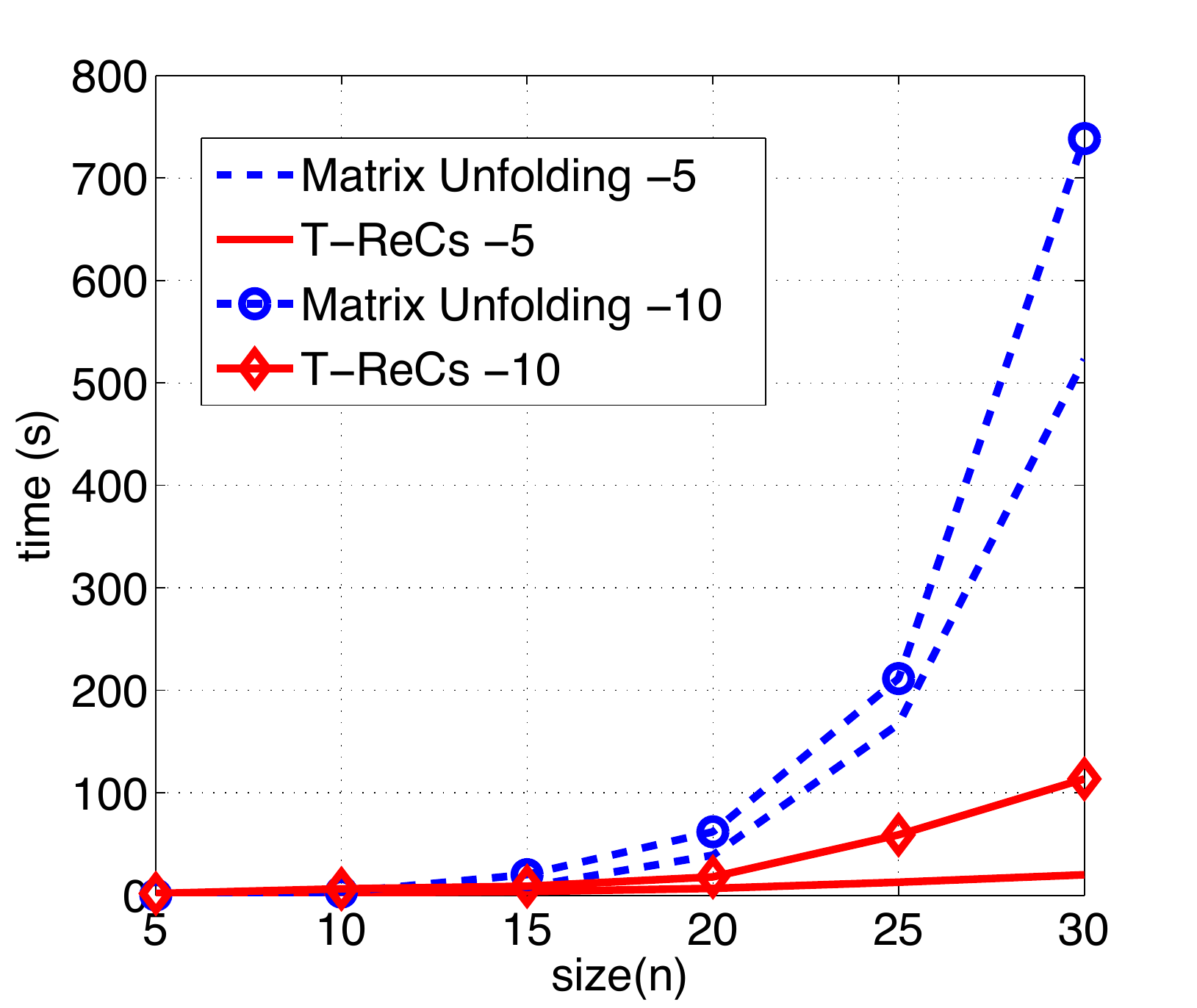}
\label{time}} \quad
\subfigure[Time taken for tensor completion by our method (T-ReCs) to that of flattening the tensor (Matrix Unfolding).]{\includegraphics[width = 60mm, height = 50mm]{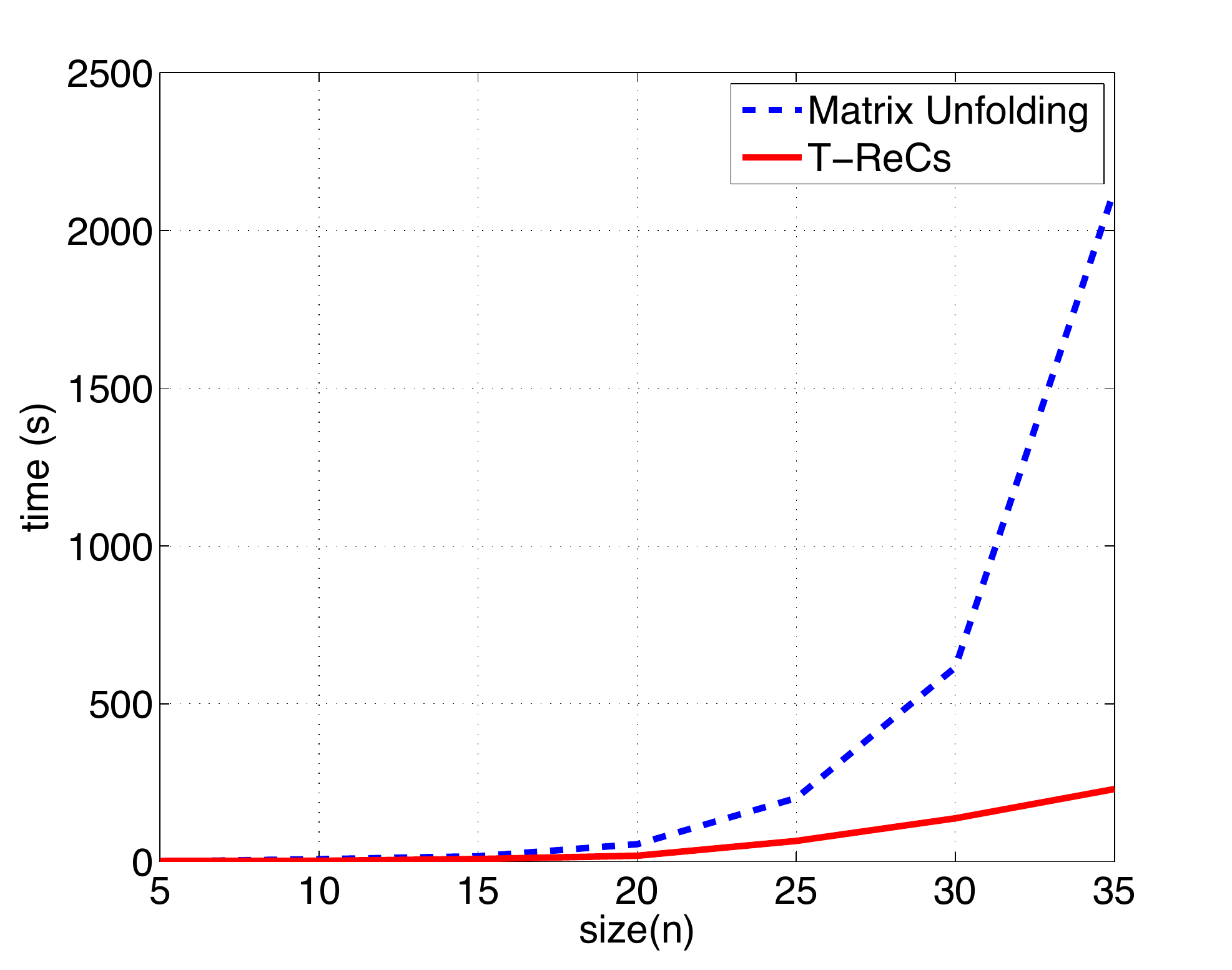}
\label{ten_time}}
\end{center}
\end{figure}

%\begin{figure}[!h]
%\centering
%%\includegraphics[width = 60mm, height = 50mm]{time.eps}
%\includegraphics[width = 60mm, height = 50mm]{ten_rec_new.eps}
%\caption{Time taken to recover third order tensors. The numbers $5$ and $10$ in the legend refer to the cases where we obtain $5n$ and $10n$ measurements respectively.} 
%%\textcolor{red}{GT: Shall we label the curves as "matricization" vs "T-ReCs"? "Matrix" vs "Tensor" is a bit confusing.}}
%\label{time}
%\end{figure}

Along lines similar to the recovery case, we compared execution times to complete a $35 \times 35 \times 35$ sized tensor. Figure \ref{ten_time} shows again that the matrix completion approach takes orders of magnitude more time than that taken by our method. We average the results over 10 independent trials, and set $r = \frac{n}{5}, ~\ m = 3nr$

%\begin{figure}[!h]
%\centering
%%\includegraphics[width = 60mm, height = 50mm]{time3.eps}
%\includegraphics[width = 60mm, height = 50mm]{time_comp_new.eps}
%\caption{Time taken for tensor completion by our method (T-ReCs) to that of flattening the tensor (Matrix Unfolding).}
%% \textcolor{red}{GT: Shall we label the curves as "matricization" vs "T-ReCs"? "Matrix" vs "Tensor" is a bit confusing. }}
%\label{ten_time}
%\end{figure}

\section{Conclusion and Future Directions} \label{sec:conclusion}
We introduced a computational framework for exact recovery of low rank tensors. A new class of measurements, known as \emph{separable} measurements was defined, and sensing mechanisms pf practical interest such as random projections and tensor completion with samples restricted to a few slices were shown to fit into the separable framework. Our algorithm, known as T-ReCs, built on the classical Leurgans' algorithm for tensor decomposition, was shown to be computationally efficient, and enjoy almost optimal sample complexity guarantees in both the random projection and the completion settings.
A number of interesting avenues for further research follow naturally as a consequence of this work:
\begin{enumerate}
\item \textbf{Robustness:} Our algorithm has been analyzed in the context of \emph{noiseless} measurements. It would be interesting to study variations of the approach and the resulting performance guarantees in the case when measurements are noisy, in the spirit of the matrix completion literature \cite{matcompnoise}.
\item \textbf{Non-separable measurements:} Our approach relies fundamentally on the measurements being separable. Tensor inverse problems, such as tensor completion in the setting when samples are obtained randomly and uniformly from the tensor do not fit into the separable framework. Algorithmic approaches for non-separable measurements thus remains an important avenue for further research.
\item \textbf{Tensors of intermediate rank:} Unlike matrices, the rank of a tensor can be larger than its (largest) dimension, and indeed increase polynomially in the dimension. The approach described in this paper addresses inverse problems where the rank is smaller than the dimension (low-rank setting). Extending these methods to the intermediate rank setting is an interesting and challenging direction for future work.
\item \textbf{Methods for tensor regularization:} Tensor inverse problems present an interesting dichotomy with regards to rank regularization. On the one hand, there is no known natural and tractable rank-regularizer (unlike the matrix case, the  nuclear norm is not known to be tractable to compute). While various relaxations for the same have been proposed, the resulting approaches (while polynomial time), are neither scalable nor known to enjoy strong sample complexity guarantees. On the other hand, matrix nuclear norm has been used in the past in conjunction with matrix unfolding, but the resulting sample complexity performance is known to be weak. Our work establishes a third approach, we bypass the need for unfolding and expensive regularization, yet achieve almost optimal sample complexity guarantees and a computational approach that is also far more scalable. However, the method applies only for the case of separable measurements. This raises interesting questions regarding the need/relevance for tensor regularizers, and the possibility to bypass them altogether.
%\item Do we need new regularizers for tensors? In what settings?
%\item Tensors of rank beyond $n$ are not organically handled here. We can probably do something along the lines of Moitra, Bhaskaran et al.
%\item Table of comparisons??
%\item Other to do: Notational consistency in higher order section, emphasize surprising simplicity of algo in intro, more discussion about separable measurements?,
\end{enumerate}

%%%%%%%%%%%%%%%%%%%%%%%%%%%%%%%%%%%%%%%%%%%%%%%%
%%%%%%%%%%%%%%%%%%%%%%%%%%%%%%%%%%%%%%%%%%%%%%%%
%%%%%%%%%%%%%%%%%%%%%%%%%%%%%%%%%%%%%%%%%%%%%%%%
%%%%%%%%%%%%%%%%%%%%%%%%%%%%%%%%%%%%%%%%%%%%%%%%
%%%%%%%%%%%%%%%%%%%%%%%%%%%%%%%%%%%%%%%%%%%%%%%%
%%%%%%%%%%%%%%%%%%%%%%%%%%%%%%%%%%%%%%%%%%%%%%%%
%%%%%%%%%%%%%%%%%%%%%%%%%%%%%%%%%%%%%%%%%%%%%%%%
%%%%%%%%%%%%%%%%%%%%%%%%%%%%%%%%%%%%%%%%%%%%%%%%
%%%%%%%%%%%%%%%%%%%%%%%%%%%%%%%%%%%%%%%%%%%%%%%%
%%%%%%%%%%%%%%%%%%%%%%%%%%%%%%%%%%%%%%%%%%%%%%%%
%\section{Conclusions and Discussion}
%\label{sec:conc}
%In this paper, we proposed a novel tensor recovery algorithm from random linear measurements. Our method is an order of magnitude faster than competing tensor recovery methods, and achieves order optimal sample complexity results. 
\bibliographystyle{plain}
\bibliography{tensor_v2}

\end{document}